\documentclass{article}

\usepackage[final,main]{neurips_2026}

\usepackage[utf8]{inputenc}
\usepackage[T1]{fontenc}
\usepackage[colorlinks=true,linkcolor={blue!75!black},citecolor={blue!75!black},urlcolor={blue!75!black}]{hyperref}
\usepackage{url}
\usepackage{booktabs}
\usepackage{amsfonts}
\usepackage{amsmath}
\usepackage{amssymb}
\usepackage{amsthm}
\usepackage{nicefrac}
\usepackage{microtype}
\usepackage{xcolor}
\usepackage{graphicx}
\usepackage{enumitem}

\newtheorem{proposition}{Proposition}
\newtheorem*{proposition*}{Proposition}

\newcommand{\arxivversion}{}

\title{Embedding Foundation Model Predictions\\in Discrete-Choice Models with Structural Guarantees}

\author{%
  Yingshuo Wang$^{1}$ \quad Xian Sun$^{2}$ \quad Yanhang Li$^{3}$ \quad Zhichao Fan$^{4}$ \quad Zexin Zhuang$^{5}$%
  \thanks{$^{1}$University of California, Berkeley, CA, USA.\quad
    $^{2}$Duke University, Durham, NC, USA.\quad
    $^{3}$Northeastern University, Boston, MA, USA.\quad
    $^{4}$University of Illinois Urbana-Champaign, IL, USA.\quad
    $^{5}$Southern Methodist University, Dallas, TX, USA.\quad
    Correspondence to: Yingshuo Wang \texttt{<yingshuow@berkeley.edu>}.\quad
    Code: \url{https://github.com/willw96/TabFM-Logit}.}%
}

\begin{document}

\maketitle

\begin{abstract}
Tabular foundation models achieve strong accuracy on choice prediction tasks, but their predictions often violate the economic logic those tasks require: raising a price can increase predicted demand, implied estimates of willingness to pay are frequently negative or implausible, and unavailable alternatives receive nonzero probability. We propose \mbox{TabFM-Logit}, an adapter that takes a foundation model's predicted choice probabilities as a precomputed feature and embeds them inside a multinomial logit's utility. In Stage 1, we fit the multinomial logit's structural coefficients by maximum likelihood with sign constraints. In Stage 2, we freeze those coefficients and fit a small neural correction operating on the foundation model's predictions. We prove that this composition exactly preserves the multinomial logit's marginal rate of substitution, so the analytically computable value of time becomes a mathematical guarantee rather than an empirical accident, and that training the two stages jointly breaks identifiability of the structural coefficients. Counterfactual derivatives and response signs are therefore the multinomial logit's by construction, and the foundation model contributes accuracy on observed choices. On respondents and households unseen in training, the adapter gains 5.7--6.8 percentage points (pp) of test accuracy on Swissmetro, 1.5 pp on LPMC, and 3.3--3.6 pp on IoT-Wearables over the multinomial logit (4.2 pp on average), stays within 1.0 pp of the raw foundation model, maintains 100\% cost monotonicity, and keeps the multinomial logit's values of time exactly. The gain is largest where the multinomial logit is weakest and smallest where it is already strong. In every case it comes with the multinomial logit's guarantees, which the raw foundation models lack.
\end{abstract}

\section{Introduction}
\label{sec:intro}

Discrete choice models guide policy decisions with significant
economic stakes: the value of time from commuter mode choice
forecasts \citep{ben_akiva_lerman_1985, train_2009} anchors the
appraisal of rail and road investments worth billions, and
estimates of willingness to pay from consumer discrete choice
experiments set prices and predict how labeling regimes shift behavior. A choice model must do two things: forecast
which alternative is chosen and how the choice responds to
intervention.

Multinomial logit utility models \citep{ben_akiva_lerman_1985,
train_2009} satisfy the second requirement by construction: a cost
coefficient constrained to be negative gives monotone demand, the
$\beta_{\text{time}}/\beta_{\text{cost}}$ ratio gives an interpretable
willingness to pay, and unavailable alternatives receive zero
probability. They satisfy the first only modestly. Modern
machine learning models invert the trade-off, raising accuracy at the
cost of breaking each structural property in turn
\citep{hillel_2021, vancranenburgh_2022, zhao_yan_2020}.

Tabular foundation models such as TabPFN, a transformer
trained on synthetic tabular tasks that predicts in context
\citep{hollmann_2023, purucker_2025}, and Mitra \citep{zhang_2024}
intensify the tension by raising the accuracy ceiling while inheriting
the same structural failures. Existing remedies each carry costs.
Monotonic neural networks \citep{sill_1997,
wehenkel_2019, constrained_monotonic_icml2025} eliminate
monotonicity violations by construction but lose the trade-off ratio in
closed form.
Knowledge distillation \citep{hinton_2015} into a multinomial logit
student preserves the structural guarantees but stays at the
multinomial logit's accuracy (\S\ref{sec:results-adapter}). Penalized
fine-tuning needs gradient access
through the foundation model's weights, incompatible with the
in-context learning that TabPFN and Mitra use at inference.

We propose \mbox{TabFM-Logit}, an adapter that embeds the foundation model's
predicted probability vector inside a constrained multinomial logit
as a precomputed feature. In Stage~1, we fit the structural
coefficients by maximum likelihood under the usual MNL constraints
with the correction held at zero. In Stage~2, we freeze those
coefficients and fit a small neural correction. We prove that
fitting in two stages preserves the structural marginal rate of
substitution exactly, and that joint training of the structural
coefficients and the correction does not: correction expressivity
creates a ridge of solutions with equal likelihood,
breaking identifiability.

Under fixed $\mathbf{q}$ the foundation model enters the
utility but not its derivatives, so the adapter preserves the
multinomial logit's counterfactual quantities rather than improving
them: trade-off ratios and response signs are the structural model's,
and the foundation model contributes accuracy on observed choices.
Letting the foundation model respond inside an intervention restores
its influence on counterfactuals and gives up the guarantee.
Section~\ref{sec:results-counterfactual} measures this trade-off. The
accuracy gain is largest where the structural model is weakest
(Swissmetro) and smallest where it is already strong (LPMC). The
guarantee is the method's main contribution, and the accuracy gain is
a measured, secondary benefit.

Our contributions are:
\begin{description}[leftmargin=1.2em,labelindent=0pt,itemsep=2pt,topsep=2pt,parsep=0pt]
\item[\mbox{TabFM-Logit}.] A behavioral adapter that preserves the
  structural multinomial logit's economic guarantees by construction
  while recovering most of a foundation model's accuracy advantage on
  observed choices.
\item[Theory.] Propositions~\ref{prop:mrs-preservation}
  and~\ref{prop:joint-collapse} identify two-stage training as the
  design choice preserving structural identifiability and characterize
  the failure mode of joint training.
\item[Evaluation.] Three discrete choice datasets and two foundation
  models, evaluated both on new choice situations and on respondents
  and households unseen in training, with bootstrap intervals. Five
  ablations, three baselines (a feature-augmented multinomial logit,
  TasteNet and knowledge distillation), calibration analysis with three
  post hoc methods, and a counterfactual evaluation of aggregate shares
  under fixed, recomputed and partial $\mathbf{q}$.
\item[Audit pipeline.] A behavioral audit recipe applicable to any
  discrete choice dataset and any prediction function, generalizing
  the audit of the value of time to willingness to pay for arbitrary
  noncost attributes and to aggregation by individual for panel
  data.
\end{description}

\section{Related work}
\label{sec:related}

\paragraph{Machine learning and economic consistency in choice modeling.}
Discrete choice modeling is the standard tool in transportation
economics and consumer behavior \citep{ben_akiva_lerman_1985}.
Recent work pushes machine learning into the field while exposing a
tension with economic structure. \citet{hillel_2021} find that
methodological pitfalls that bias performance estimates are common in
published machine learning studies of mode choice and that, without a
behavioral model, the consistency of their predictions is hard to check.
\citet{zhao_yan_2020} show that a random forest predicts mode choice
better than logit models but gives behaviorally unreasonable
elasticities and marginal effects when these are computed in the
standard way. \citet{vancranenburgh_2022} pose the integration of
machine learning with choice models grounded in behavioral theory as an
open research question.
\citet{han_2022} (TasteNet) propose a discrete choice model that
learns taste parameters as neural network functions of individual
characteristics, targeting interpretable taste parameters through
learned heterogeneity rather than the decomposition into structure
plus correction that we adopt here.
\ifdefined\arxivversion
This work builds on our workshop paper
\citep{wang_2026_workshop}. The present paper extends it with an
additional dataset, formal preservation results, and a counterfactual
evaluation of aggregate shares.
\fi
Our work additionally derives a formal preservation result,
evaluates across foundation models, and characterizes joint
training's failure mode. The two compose rather
than compete: a TasteNet can serve as Stage~1, and because the
correction adds nothing to any attribute derivative under fixed
$\mathbf{q}$, the argument of Proposition~\ref{prop:mrs-preservation}
then preserves each individual's trade-off ratio.

\paragraph{Monotonic neural networks.}
Architectural approaches enforce monotonicity through weights with
sign constraints and monotone activations \citep{sill_1997,
constrained_monotonic_icml2025} or by integrating a network with positive
output \citep{wehenkel_2019}, eliminating violations of cost
monotonicity but losing the trade-off ratio in closed form: with no
global coefficients, the analogue of the value of time must be computed
from gradients per observation rather than read off as
$\beta_{\text{time}} / \beta_{\text{cost}}$. Our adapter takes a
different point: monotonicity and trade-off ratios come from the
structural component, whose constrained parameterization is preserved
throughout training, while the foundation model contributes a
nondifferentiable side channel.

\paragraph{Tabular foundation models and knowledge distillation.}
Tabular foundation models such as TabPFN \citep{hollmann_2023,
purucker_2025} and Mitra \citep{zhang_2024} achieve strong
classification accuracy via in-context learning. Their published
benchmarks emphasize predictive accuracy and do not include
diagnostics of structural validity specific to choice modeling
(monotone responses to cost, finite positive trade-off ratios, zero
probability on unavailable alternatives) on every test row. Our audit fills this gap. Knowledge distillation
\citep{hinton_2015} into a multinomial logit student preserves
structural guarantees but is bounded by the student's expressive
capacity: a plain MNL student parameterizes a utility
strictly linear in the features, and the residual variance the foundation
model captures lives above that ceiling. We position the foundation
model differently: its predictions
become an explanatory feature embedded inside a structurally
constrained utility, with the foundation model's parameters never
modified.

\section{Method}
\label{sec:method}

\subsection{Setup and notation}
\label{sec:setup}
We consider a discrete choice setting with $N$ observations indexed by
$i \in \{1, \dots, N\}$. Each observation has a feature vector
$\mathbf{x}_i \in \mathcal{X}$, a set of available alternatives
$\mathcal{K}_i \subseteq \{1, \dots, K\}$, and an observed choice
$y_i \in \mathcal{K}_i$. For panel data we additionally observe a
subject identifier $s_i$. Multiple observations from the same
subject are not exchangeable, so we cluster on $s_i$ throughout.

A choice model produces probabilities $P_k(\mathbf{x}_i) \in [0, 1]$
summing to one over $\mathcal{K}_i$, i.e., a vector on the
$(K{-}1)$-simplex $\Delta^{K-1}$. The standard multinomial logit
(MNL) parameterizes these through a linear utility
$V_k(\mathbf{x}_i) = \boldsymbol{\beta}^\top
\boldsymbol{\phi}_k(\mathbf{x}_i)$ and the softmax
$P_k = \exp V_k / \sum_{j \in \mathcal{K}_i} \exp V_j$
\citep{ben_akiva_lerman_1985}. The feature
transform $\boldsymbol{\phi}_k$ selects from $\mathbf{x}_i$ the columns
that enter alternative $k$'s utility, typically including cost, time,
alternative-specific constants, and sociodemographic interactions.
We treat $\boldsymbol{\phi}_k$ as fixed by the dataset's
specification and recover $\boldsymbol{\beta}$ by maximum likelihood.

\subsection{Behavioral audit}
\label{sec:audit}
Every model is evaluated through three functionals
that take only a prediction function $\mathbf{x} \mapsto
P(\mathbf{x}) \in \Delta^{K-1}$ as input. They read model outputs
only and need no gradient access or knowledge of internals.

\textbf{Intervention protocol.} For the adapter and the
feature-augmented MNL we use the protocol with \emph{fixed} $\mathbf{q}$:
when a cost or attribute is perturbed at row $i$, the foundation model's
probability vector $\mathbf{q}_i$ is held at its value on the
unperturbed $\mathbf{x}_i$, and the perturbation enters only
$V^{\text{struct}}_k$. The protocol with \emph{recomputed} $\mathbf{q}$
reruns the foundation model on the perturbed input. We use it for
the raw foundation model, and for the adapter in
Section~\ref{sec:results-counterfactual}. At deployment, $\mathbf{q}$
is computed once for each observed input, including inputs from
individuals absent from training, and
Proposition~\ref{prop:mrs-preservation} holds at that input exactly
as at a training row. Structural queries (trade-off ratios, responses
to a cost change) are answered under fixed $\mathbf{q}$. A
counterfactual in which the foundation model itself responds to the
perturbation is a different quantity, without the guarantee, and
Section~\ref{sec:results-counterfactual} evaluates it.

The three functionals are:
(1) \emph{Monotonicity.} For each test row $i$ and alternative $k$,
perturb $k$'s cost upward by $1\%$ of its observed range and check
whether $P_k$ falls. We report the rate over observations
(aggregated by subject on panel data).
(2) \emph{Trade-off ratio.} The marginal rate of substitution between
noncost attribute $a$ and cost $b$, reported in the standard
sign convention of transportation economics so that value of time and
willingness to pay for attributes that raise utility are positive:
$\rho^{\mathrm{VOT}}_{a,b} = (\partial P_k/\partial a) / (\partial
P_k/\partial b)$ when both partials are negative (e.g., time and
cost both lower utility), and $\rho^{\mathrm{WTP}}_{a,b} =
-(\partial P_k/\partial a) / (\partial P_k/\partial b)$ when
$\partial P_k/\partial a$ is positive and $\partial P_k/\partial b$
is negative (a desirable noncost attribute vs cost). Estimated by
finite differences scaled to $1\%$ of each column's observed range.
(3) \emph{Availability compliance.} For datasets in which the
available set $\mathcal{K}_i$ varies across observations (e.g.,
Swissmetro, where the proposed Swissmetro option is not offered to
some respondents), $\mathrm{Leak}(M) = \mathbb{E}_i
[\sum_{k \notin \mathcal{K}_i} P_k(\mathbf{x}_i) \mid \mathcal{K}_i
\neq \{1, \dots, K\}]$ is the mean predicted probability assigned to
formally unavailable alternatives, over observations that have
one. The multinomial logit applies softmax over
$\mathcal{K}_i$ only, so its leak is mechanically zero. Black box
predictors that ignore the availability mask can leak nonzero
probability onto unavailable alternatives.

\subsection{Behavioral adapter}
\label{sec:adapter-arch}
\begin{description}[leftmargin=1.2em,labelindent=0pt,itemsep=2pt,topsep=2pt,parsep=0pt]
\item[Architecture.] For each observation $i$, alternative $k$, and a
  precomputed probability vector
  $\mathbf{q}(\mathbf{x}_i) \in \Delta^{K-1}$ obtained by a single
  forward pass through the foundation model on the raw input, the
  utility of the adapter, \mbox{TabFM-Logit}, is
\begin{equation}
  V_k(\mathbf{x}_i) =
  \underbrace{\boldsymbol{\beta}^\top \boldsymbol{\phi}_k(\mathbf{x}_i)}_{V_k^{\text{struct}}\text{ (economic structure)}}
  + \underbrace{g_k(\mathbf{q}(\mathbf{x}_i))}_{\text{foundation model correction}},
  \label{eq:adapter}
\end{equation}
where $g \colon \Delta^{K-1} \to \mathbb{R}^K$ is a small MLP (two
hidden layers, width $32$). To enforce cost/time monotonicity
throughout training we reparameterize each such coefficient as
$\beta = -\exp(\theta)$.

\item[Training.] We fit the model in two sequential stages.
  \emph{Stage 1}: with the correction held at $g \equiv 0$, fit
  $\boldsymbol{\beta}$ by maximum likelihood under the sign
  constraints, recovering the standalone MNL estimate
  $\boldsymbol{\beta}^\ast$. \emph{Stage 2}: fix
  $\boldsymbol{\beta} = \boldsymbol{\beta}^\ast$ and fit only $g$ by
  maximum likelihood. The architecture realizes Stage~1 cleanly
  because $g$'s output layer is initialized at zero. Hidden weights keep
  PyTorch's default initialization. Because
  $\mathbf{q}(\mathbf{x}_i)$ is precomputed,
  $\partial \mathbf{q}/\partial \mathbf{x}_i = 0$, so monotonicity
  and trade-off ratios reduce to functions of
  $\boldsymbol{\beta}^\ast$ alone.
  Propositions~\ref{prop:mrs-preservation}--\ref{prop:joint-collapse}
  make this precise and identify joint training as the failure mode.
\end{description}

%
%

\subsection{Propositions}
\label{sec:propositions}
We use $\boldsymbol{\beta}^\ast$ for the Stage~1 maximum likelihood
estimate and reuse the notation of \eqref{eq:adapter}.

\begin{proposition}[Preservation of the marginal rate of substitution
  by construction, fixed $\mathbf{q}$]
\label{prop:mrs-preservation}
Let $\boldsymbol{\beta}^\ast$ be the Stage~1 maximum likelihood
estimate, let $g$ be any Stage~2 parameters in \eqref{eq:adapter},
and hold $\mathbf{q}$ fixed as in
Section~\ref{sec:audit}. For any two attributes $j, j'$ that enter the
model only through
$V^{\text{struct}}_k(\mathbf{x}_i) = \boldsymbol{\beta}^\top
\boldsymbol{\phi}_k(\mathbf{x}_i)$, and for any $\mathbf{x}_i$ at
which $\boldsymbol{\phi}_k$ is differentiable in $x_{ij}$ and
$x_{ij'}$,
\begin{equation*}
  \mathrm{MRS}_{j,j'}(\mathbf{x}_i)
  \;\equiv\;
  \frac{\partial V_k(\mathbf{x}_i) / \partial x_{ij}}{\partial V_k(\mathbf{x}_i) / \partial x_{ij'}}
  \;=\;
  \frac{\beta^\ast_j}{\beta^\ast_{j'}}.
\end{equation*}
In particular, the analogue of the value of time,
$\beta^\ast_{\text{time}} / \beta^\ast_{\text{cost}}$, is identical to
the standalone multinomial logit's value of time and is invariant to
the choice of foundation model and to the choice of $g$. When
$\mathbf{q}$ is recomputed, $g$ contributes a chain rule term through
$\mathbf{q}$ to each partial and the ratio no longer reduces to
$\beta^\ast_j / \beta^\ast_{j'}$ in general.
\end{proposition}

\begin{itemize}[leftmargin=2em,itemsep=2pt,topsep=2pt,parsep=0pt]
\item \textbf{Corollary (MRS from probability derivatives).}
  \label{cor:prob-deriv-mrs}
  With $\mathbf{q}$ fixed as above, restrict attention to
  attributes $j, j'$ that enter only alternative $k$'s utility.
  Then the audit estimator equals the structural coefficient ratio
  under the audit's sign convention: $\rho^{\mathrm{VOT}}_{j,j'}
  (\mathbf{x}_i) = \beta^\ast_j / \beta^\ast_{j'}$ and
  $\rho^{\mathrm{WTP}}_{j,j'}(\mathbf{x}_i) = -\beta^\ast_j /
  \beta^\ast_{j'}$, since the softmax derivative's $P_k(1-P_k)$
  factor cancels in the ratio (Appendix~\ref{sec:proof-prop1}).
\item \textbf{Intuition.}
  Fixed $\mathbf{q}$ forces $\partial \mathbf{q}/\partial
  \mathbf{x} = 0$, so $g$ vanishes from $\partial V_k / \partial
  x_{ij}$ and the ratio collapses to $\beta^\ast_j /
  \beta^\ast_{j'}$. Full proof in Appendix~\ref{sec:proof-prop1}.
\end{itemize}

\begin{proposition}[Joint training breaks structural identifiability]
\label{prop:joint-collapse}
Assume that (i) there exists a continuous $\kappa_k: \Delta^{K-1} \to
\mathbb{R}$ such that $\mathrm{cost}_k(\mathbf{x}) =
\kappa_k(\mathbf{q}(\mathbf{x}))$ on the closure of the training
support (\emph{cost recoverability}), (ii) the correction
class $\mathcal{G}$ is dense in
$C^0(\Delta^{K-1}, \mathbb{R}^K)$~\citep{cybenko1989, leshno1993}, and (iii)
$L(\boldsymbol{\beta}, g)$ is minimized jointly rather than in two
stages. Then for any joint minimizer
$(\boldsymbol{\beta}^{(0)}, g^{(0)})$ there exists a
family $\{(\boldsymbol{\beta}^{(c)}, g^{(c)}) : c \in \mathcal{C}\}$
of distinct configurations achieving the same loss (exactly if
$\mathcal{G}$ contains the required corrections, to arbitrary
precision otherwise), parametrized
by $\beta^{(c)}_{\text{cost}} = \beta^{(0)}_{\text{cost}} + c$ for
$c$ in an open interval $\mathcal{C}$ preserving sign constraints
on $\beta_{\text{cost}}$. Gradient descent within $\mathcal{C}$
selects a point that depends on initialization rather than the MNL MLE
$\boldsymbol{\beta}^\ast$.
\end{proposition}

\begin{itemize}[leftmargin=2em,itemsep=2pt,topsep=2pt,parsep=0pt]
\item \textbf{Note (cost recoverability).}
  The assumption depends on the dataset. Appendix~\ref{sec:prop2-partial}
  covers partial recoverability.
\item \textbf{Intuition.}
  Subtracting $c \cdot \kappa_k(\mathbf{q})$ from $g_k$ shifts
  $\beta_{\text{cost}}$ by $+c$ without changing pointwise utility:
  the structural increase $c \cdot \mathrm{cost}_k$ in $V_k$ is
  cancelled by the matching decrease in the correction. A3
  (\S\ref{sec:ablations}) is the empirical illustration. Full proof
  in Appendix~\ref{sec:proof-prop2}.
\end{itemize}

\section{Experimental setup}
\label{sec:setup-experiments}

\paragraph{Datasets.}
We evaluate on three discrete choice datasets. Swissmetro
\citep{bierlaire_2001}: $10{,}719$ intercity travel choices (stated preference)
among rail, the proposed Swissmetro, and car. LPMC
\citep{hillel_2018}: $81{,}086$ London trips (revealed preference) among
walk, cycle, public transport, and drive. IoT-Wearables
\citep{iot_wearables_dce_2020}: $6{,}362$ choices (stated preference)
among three Internet of Things wearable devices varying in price,
functional features, and a security/privacy labeling scheme (panel
data, $728$ subjects). All splits are $70/15/15$, and we evaluate two
prediction targets. \emph{Row-level splits} assign choice observations
at random, stratified by the chosen alternative. Because respondents
answer repeated tasks (nine each on Swissmetro) and people make
several trips (LPMC), every Swissmetro test row and $80.5\%$ of
LPMC test rows then come from people who also appear in training, so
these splits measure prediction of new choice situations for people
already observed. \emph{Grouped splits} measure prediction for new
individuals: no respondent (Swissmetro, $834/179/178$ respondents) or
household (LPMC, $12{,}331/2{,}642/2{,}643$ households) appears in
more than one split, using the identifiers in the raw data files.
Both split types use the same rows and preprocessing. IoT-Wearables
is split by subject (each subject's rows go entirely to one split), so
its two split types coincide. Table~\ref{tab:headline} and the
protocol analyses use the row-level splits, and Table~\ref{tab:grouped}
the grouped splits.

\paragraph{Models.}
Table~\ref{tab:headline} compares six models: the multinomial logit
(Stage~1 of the adapter), the raw foundation model, a monotonic
neural network with nonnegative weights and ReLU activations on cost
\citep{constrained_monotonic_icml2025}, TasteNet \citep{han_2022}, a
feature-augmented multinomial logit, and our primary model, the
adapter $V_k = V^{\text{struct}}_k + g_k(\mathbf{q})$ of
\eqref{eq:adapter}. The foundation model is Mitra \citep{zhang_2024}
or TabPFN \citep{hollmann_2023, purucker_2025}. TasteNet is a neural
network that maps individual characteristics to taste coefficients
with sign constraints
(Appendix~\ref{sec:baselines-extra}). The feature-augmented model adds
$\alpha \log q_k$ to each alternative's utility and fits all
coefficients jointly by maximum likelihood.
Three further variants appear in prose: a masked foundation model
(Swissmetro, \S\ref{sec:results-fm-fails}), a convex ensemble
$\alpha P_{\text{MNL}} + (1{-}\alpha) P_{\text{FM}}$
(\S\ref{sec:results-adapter}), and knowledge distillation of each
foundation model into the same constrained multinomial logit
(\S\ref{sec:results-adapter}, Appendix~\ref{sec:baselines-extra}).
The correction network $g$ is an MLP with two hidden layers of width
$32$ and an output layer initialized at zero (\S\ref{sec:adapter-arch}). Stage~1
and the feature-augmented multinomial logit are fit by maximum
likelihood to convergence (L-BFGS, then Newton steps, until the largest
gradient entry is below $10^{-8}$). Stage~2 uses Adam on the full batch
(learning rate $0.005$) until the validation NLL has not improved for
$30$ epochs, keeping the parameters with the lowest validation NLL.

\paragraph{Foundation model inputs and cross-fitted training $\mathbf{q}_i$.}
Cost (or price), time (or duration) and availability indicators are
among the foundation model inputs on every dataset (full lists in
Appendix~\ref{sec:cross-fit-details}), so the foundation model can
respond nonmonotonically to cost, which is why we hold $\mathbf{q}$
fixed (Section~\ref{sec:audit}).
The $\mathbf{q}_i$ on training rows used by Stage~2 and by the
feature-augmented multinomial logit come from stratified cross-fitting
with $k=5$ folds, so no row's prediction was made by a model
that saw its own label. Test $\mathbf{q}_i$ are out of context by
construction. On the grouped splits the folds are themselves grouped,
so no training row's $\mathbf{q}_i$ is informed by the same person's
or household's other rows. Appendix~\ref{sec:cross-fit-details}
compares cross-fitted $\mathbf{q}_i$ on the training rows with
$\mathbf{q}_i$ fitted in sample, and adapter test accuracy differs by
at most $0.3$ pp. LPMC's
context (training plus validation rows, about $69{,}000$) is outside TabPFN's supported size
range, so TabPFN is evaluated on Swissmetro and IoT-Wearables and
Mitra on all three datasets. Validation $\mathbf{q}_i$, used for
early stopping and for fitting the post hoc calibrators, come from a
fit whose context contains the validation rows.
Section~\ref{sec:discussion} discusses the effect on Swissmetro
calibration.

\paragraph{Metrics and uncertainty.}
We report test accuracy, the monotonicity rate per row (aggregated
by subject on IoT-Wearables), the trade-off ratio, the availability
leak (Swissmetro only), and expected calibration error (ECE, $15$
confidence bins of equal width) after post hoc calibration by
scalar temperature scaling \citep{guo_2017}, vector temperature
scaling or isotonic regression, each fit on validation and evaluated
on test (Table~\ref{tab:headline} reports the best of the three). The
trade-off ratio is the value of time on the transportation datasets
and the willingness to pay for the function and label attributes on
IoT-Wearables (for TasteNet, the median over test individuals).
For each (dataset, foundation model) cell we draw $10$ bootstrap
samples of the training set with replacement and refit Stage~1 and
Stage~2 independently on each replicate. Validation and test splits
are held fixed. We report the accuracy gain over the multinomial
logit with two bootstrap $95\%$ intervals ($2{,}000$ resamples) that
capture different uncertainty: an interval across replicates
(training variability on the fixed test set) and an interval from
resampling the test set (finite test set), which resamples subjects
on IoT-Wearables and respondents or households on the grouped splits.
On the row-level Swissmetro and LPMC splits it resamples rows. Rows
from one person are not independent, so these intervals understate
uncertainty.

\section{Main results}
\label{sec:results}

\begin{table*}[!htbp]
\centering
\footnotesize
\setlength{\tabcolsep}{4pt}
\renewcommand{\arraystretch}{0.95}
\caption{\textbf{Headline results} across three datasets and two foundation models
(Mitra, TabPFN): row-level splits for Swissmetro and LPMC, the split by subject for
IoT-Wearables (grouped splits in Table~\ref{tab:grouped}). Means over $10$
bootstrap replicates of the training set ($5$ random initializations for TasteNet). For feat-aug the trade-off
is the median over replicates, and the superscript counts replicates whose cost coefficient is at
its zero bound (ratio undefined). Trade-off: value of time in CHF/hr
(Swissmetro) and GBP/hr (LPMC, public transport / drive), willingness to pay in USD
(IoT-Wearables, function / label). Monotonic NN and TasteNet use no foundation model,
so their two columns coincide. The monotonic NN has no trade-off ratio in closed form.
$^{\ddagger}$~Not reported (LPMC exceeds TabPFN's size range). \textbf{Bold} marks
failures of behavioral validity: monotonicity below $100\%$, a trade-off ratio with
the wrong sign, an implausible magnitude, a spread across replicates as large as
its mean or a collapsed cost coefficient, or accuracy far below the multinomial logit.}
\label{tab:headline}
\resizebox{\textwidth}{!}{%
\begin{tabular}{llcccccccc}
\toprule
 & & \multicolumn{2}{c}{Acc.\ (\%)} & \multicolumn{2}{c}{Mono.\ (\%)} & \multicolumn{2}{c}{Trade-off} & \multicolumn{2}{c}{ECE (\%)} \\
\cmidrule(lr){3-4} \cmidrule(lr){5-6} \cmidrule(lr){7-8} \cmidrule(lr){9-10}
Dataset & Model & Mitra & TabPFN & Mitra & TabPFN & Mitra & TabPFN & Mitra & TabPFN \\
\midrule
Swissmetro    & MNL          & 63.6 & 63.6 & 100 & 100             & 91.6          & 91.6          & 4.3 & 4.3 \\
              & raw FM       & 77.7 & 76.2 & --- & \textbf{96.4}   & ---           & 72.5          & 9.4 & 8.7 \\
              & monotonic NN & 63.1 & 63.1 & 100 & 100             & ---           & ---           & --- & --- \\
              & TasteNet     & 70.2 & 70.2 & 100 & 100             & 73.8          & 73.8          & 3.6 & 3.6 \\
              & feat-aug     & 77.5 & 76.2 & 100 & 100             & \textbf{83.8}$^{1}$ & \textbf{1.2}$^{9}$ & 9.2 & 8.9 \\
              & adapter      & 76.2 & 74.9 & 100 & 100             & 91.6          & 91.6          & 7.6 & 6.4 \\
\midrule
LPMC          & MNL          & 70.8 & 70.8 & 100 & 100             & 19.8/39.0    & 19.8/39.0    & 1.8 & 1.8 \\
              & raw FM       & 74.2 & $^{\ddagger}$ & --- & $^{\ddagger}$ & --- & $^{\ddagger}$ & 0.9 & $^{\ddagger}$ \\
              & monotonic NN & \textbf{52.6} & \textbf{52.6} & 100 & 100 & ---    & ---          & --- & --- \\
              & TasteNet     & 72.4 & 72.4 & 100 & 100             & 8.3/25.1     & 8.3/25.1     & 1.3 & 1.3 \\
              & feat-aug     & 74.0 & $^{\ddagger}$ & 100 & $^{\ddagger}$ & 6.6/11.6 & $^{\ddagger}$ & 0.9 & $^{\ddagger}$ \\
              & adapter      & 72.9 & $^{\ddagger}$ & 100 & $^{\ddagger}$ & 19.8/39.0 & $^{\ddagger}$ & 1.6 & $^{\ddagger}$ \\
\midrule
IoT-Wearables & MNL          & 62.8 & 62.8 & 100             & 100             & 42.3/15.1           & 42.3/15.1                & 2.5 & 2.5 \\
              & raw FM       & 67.1 & 67.0 & \textbf{45.6}   & \textbf{40.8}   & \textbf{0.0/-0.1}   & \textbf{-10.5/-8.8}      & 4.2 & 4.3 \\
              & monotonic NN & 64.0 & 64.0 & 100             & 100             & ---                 & ---                      & --- & --- \\
              & TasteNet     & 63.4 & 63.4 & 100             & 100             & 51.6/16.3           & 51.6/16.3                & 4.0 & 4.0 \\
              & feat-aug     & 67.3 & 66.9 & 100             & 100             & \textbf{20.9/8.5}$^{1}$ & \textbf{30.7/411.0}$^{5}$ & 4.2 & 4.3 \\
              & adapter      & 66.1 & 66.4 & 100             & 100             & 42.3/15.1           & 42.3/15.1                & 2.7 & 3.1 \\
\bottomrule
\end{tabular}%
}
\end{table*}

Table~\ref{tab:headline} reports test accuracy, monotonicity
rate, trade-off ratio, and expected calibration error after calibration
(ECE) in each cell for the six models of \S\ref{sec:setup-experiments}, and
Table~\ref{tab:grouped} reports accuracy on new individuals. We take
in turn the raw foundation models' failures
(\S\ref{sec:results-fm-fails}), the adapter's recovery of behavioral
validity (\S\ref{sec:results-adapter}), and counterfactual aggregate
shares (\S\ref{sec:results-counterfactual}).

\subsection{The foundation models fail behavioral validity}
\label{sec:results-fm-fails}

Mitra and TabPFN gain $+3$ to $+14$ pp of accuracy over the
multinomial logit on the row-level splits, but the gains come with
three distinct failures of behavioral validity.
\textbf{Monotonicity} fails most on the dataset with the most
predictive headroom per alternative. On IoT-Wearables, TabPFN is
monotone in price on only $40.8\%$ of test rows and Mitra on
$45.6\%$. On Swissmetro TabPFN reaches $96.4\%$, still below the
multinomial logit's mathematical $100\%$. \textbf{Trade-off ratios}
compound the issue: TabPFN's willingness to pay on IoT-Wearables has
the wrong sign for both attributes ($-10.5$ and $-8.8$ USD against the
multinomial logit's $+42.3$ and $+15.1$). Mitra's willingness to pay
on the binary IoT-Wearables indicators is locally inconclusive at the
$1\%$ perturbation scale, and flipping the indicators gives
small values of both signs (Appendix~\ref{sec:audit-discrete}). \textbf{Availability compliance}
fails on Swissmetro: over the $253$ test rows with an unavailable
alternative, TabPFN places on average $6 \times 10^{-4}$ and Mitra
$2 \times 10^{-3}$ of probability on it, against at most
$4 \times 10^{-10}$ for the multinomial logit and the adapter, whose
availability mask makes the leak zero up to numerical precision. Masking the foundation model (zeroing
unavailable alternatives and renormalizing) removes the leak but
leaves accuracy unchanged, so the leak is a structural error rather
than an artifact of probability redistribution.

\subsection{The adapter recovers behavioral validity while keeping most of the accuracy gain}
\label{sec:results-adapter}

\begin{table}[!htbp]
\centering
\footnotesize
\setlength{\tabcolsep}{4pt}
\renewcommand{\arraystretch}{0.95}
\caption{\textbf{Accuracy on new individuals}: test accuracy (\%) when no respondent
(Swissmetro), household (LPMC) or subject (IoT-Wearables) appears in more than one
split, with $\mathbf{q}$ on training rows cross-fitted in folds that never split a group.
IoT-Wearables is split by subject throughout, so its columns match
Table~\ref{tab:headline}. Means over $10$
bootstrap replicates of the training set. Brackets are $95\%$ bootstrap intervals that resample
test respondents, households or subjects ($2{,}000$ resamples). TabPFN is not
reported on LPMC (context outside its size range).}
\label{tab:grouped}
\begin{tabular}{lccccc}
\toprule
 & \multicolumn{2}{c}{Swissmetro} & LPMC & \multicolumn{2}{c}{IoT-Wearables} \\
\cmidrule(lr){2-3} \cmidrule(lr){4-4} \cmidrule(lr){5-6}
 & Mitra & TabPFN & Mitra & Mitra & TabPFN \\
\midrule
MNL                   & 64.8 & 64.8 & 71.7 & 62.8 & 62.8 \\
Raw foundation model  & 69.8 & 72.5 & 73.8 & 67.1 & 67.0 \\
Feature-augmented MNL & 71.6 & 72.6 & 73.8 & 67.3 & 66.9 \\
Adapter               & 70.5 & 71.6 & 73.2 & 66.1 & 66.4 \\
\midrule
Adapter $-$ MNL (pp)    & $+5.7$ & $+6.8$ & $+1.5$ & $+3.3$ & $+3.6$ \\
                        & {\scriptsize$[+2.6, +9.2]$} & {\scriptsize$[+3.7, +10.1]$} & {\scriptsize$[+0.9, +2.1]$} & {\scriptsize$[+1.4, +5.2]$} & {\scriptsize$[+1.7, +5.6]$} \\
Adapter $-$ raw FM (pp) & $+0.7$ & $-0.9$ & $-0.5$ & $-1.0$ & $-0.6$ \\
                        & {\scriptsize$[-2.1, +3.5]$} & {\scriptsize$[-2.9, +1.0]$} & {\scriptsize$[-1.1, +0.007]$} & {\scriptsize$[-1.8, -0.4]$} & {\scriptsize$[-1.1, +0.1]$} \\
\bottomrule
\end{tabular}
\end{table}

The adapter inherits the multinomial logit's structural utility
(with a sign constraint on the cost coefficient) and availability
mask, so it satisfies $100\%$ monotonicity, exact analytical
trade-off ratios, and zero availability leak by construction on
every cell. The empirical question is how much of the
foundation model's accuracy gain the adapter retains.

The adapter recovers most of it. On the row-level splits it trails
the raw foundation model by $0.6$ to $1.5$ pp
(Table~\ref{tab:headline}), the price of full preservation of the
multinomial logit's economic guarantees. With the test set resampled, the
$95\%$ interval of its gain over the multinomial logit is $[+10.4, +14.8]$
and $[+9.2, +13.5]$ pp on Swissmetro (Mitra, TabPFN), $[+1.7, +2.5]$
on LPMC, and $[+1.4, +5.2]$ and $[+1.7, +5.6]$ on IoT-Wearables (rows
resampled on Swissmetro and LPMC, subjects on IoT-Wearables,
intervals across replicates in Appendix~\ref{sec:uncertainty}). On new
individuals (Table~\ref{tab:grouped}) the gain is $+5.7$ and $+6.8$
pp on Swissmetro and $+1.5$ pp on LPMC, every interval excludes zero,
and the adapter lies between $1.0$ pp below and $0.7$ pp above the
raw foundation model. The trade-off ratios match the structural
multinomial logit by construction
(Proposition~\ref{prop:mrs-preservation}) and are stable across
replicates: $91.6 \pm 4.7$ CHF/hr on Swissmetro, $19.84 \pm 0.30$ and
$38.96 \pm 0.70$ GBP/hr on LPMC public transport and driving, and
$+42.28 \pm 1.91$ USD willingness to pay for functional features on
IoT-Wearables.

Two other models that keep structural guarantees reach lower
accuracy than the adapter. Distilling each foundation model into the same constrained
multinomial logit stays within $0.3$ pp of the plain multinomial
logit on every cell: distillation changes the training target, not
the hypothesis class. TasteNet, which maps individual characteristics
to tastes with sign constraints, lands between the multinomial logit
and the adapter on every dataset ($70.2$, $72.4$ and $63.4\%$). It keeps
monotonicity but replaces the population trade-off ratio with one per
individual (Swissmetro median $73.8$ CHF/hr). Details are in
Appendix~\ref{sec:baselines-extra}.

The feature-augmented multinomial logit reaches slightly higher
accuracy than the adapter ($+0.5$ to $+1.3$ pp depending on the cell,
paired intervals in Appendix~\ref{sec:uncertainty}), but its
structural coefficients do not survive the joint fit. At its
maximum likelihood estimate the cost coefficient falls to $0$--$17\%$
of the multinomial logit's (median by cell) and reaches its zero bound,
where the trade-off ratio is undefined, in $16$ of the $50$ replicates
($9$ of $10$ on Swissmetro with TabPFN): the $\alpha \log q_k$ term
absorbs the cost variation that $\mathbf{q}$ already encodes, as in
Proposition~\ref{prop:joint-collapse}. The adapter, by contrast,
inherits the multinomial logit's trade-off ratios exactly under
fixed $\mathbf{q}$ (Proposition~\ref{prop:mrs-preservation}). The convex ensemble baseline,
$\alpha\, P_{\text{MNL}} + (1 - \alpha)\, P_{\text{FM}}$ with $\alpha$
fit on validation, selects $\alpha \le 0.03$ on every cell and so
reduces to the raw foundation model (accuracy within $0.02$ pp). Both
baselines confirm that linear blending does not
recover the adapter's combination of accuracy and interpretability.

\paragraph{Calibration.}
ECE after calibration is at or below the multinomial logit's on LPMC
and within $0.6$ pp of it on IoT-Wearables (Table~\ref{tab:headline}). Swissmetro is the
exception: the adapter's ECE stays at $7.6\%$ (Mitra) and $6.4\%$
(TabPFN) against the multinomial logit's $4.3\%$ after scalar or
vector temperature scaling or isotonic regression, and bootstrap
ensembling does not close the gap. Section~\ref{sec:discussion} traces
the gap to the validation inputs of the calibrators. Full
NLL, Brier, and uncalibrated ECE are in Appendix~\ref{sec:cal-metrics}.

\subsection{Counterfactual evaluation of aggregate shares}
\label{sec:results-counterfactual}

To test how each model would predict aggregate demand response to a
small price increase, we evaluate every model on the test
set under a $+10\%$ cost perturbation applied separately to each
alternative, then aggregate predicted probabilities into a market
share for the perturbed alternative.
Table~\ref{tab:counterfactual} reports the full result matrix.

\begin{table}[h]
  \centering
  \footnotesize
  \setlength{\tabcolsep}{6pt}
  \renewcommand{\arraystretch}{0.95}
  \caption{Counterfactual change in aggregate share (pp) under a $+10\%$
  cost perturbation, row-level splits. Adapter: fixed $\mathbf{q}$, mean
  over $10$ replicates (standard deviation at most $0.3$ pp). Raw
  foundation model: refit once per cell and run on the perturbed input,
  shared with Table~\ref{tab:frontier} and
  Appendix~\ref{sec:frontier-details}. \textbf{Bold} marks
  monotonicity violations ($\Delta$share $>0$ under a cost
  \emph{increase}). $^{\ddagger}$~TabPFN not reported on LPMC.}
  \label{tab:counterfactual}
  \begin{tabular}{llcccccc}
    \toprule
                  &           & Stage~1 & \multicolumn{2}{c}{Adapter} & \multicolumn{2}{c}{raw FM} \\
    \cmidrule(lr){3-3} \cmidrule(lr){4-5} \cmidrule(lr){6-7}
    Dataset       & Scenario  & MNL & Mitra & TabPFN & Mitra & TabPFN \\
    \midrule
    Swissmetro    & train     & $-0.41$ & $-0.20$ & $-0.20$ & $-1.89$ & $-1.69$ \\
                  & SM        & $-1.56$ & $-0.94$ & $-0.99$ & $-6.33$ & $-3.44$ \\
                  & car       & $-1.10$ & $-0.64$ & $-0.69$ & $-1.27$ & $-1.54$ \\
    LPMC          & PT        & $-0.33$ & $-0.27$ & $^{\ddagger}$ & $-0.06$ & $^{\ddagger}$ \\
                  & drive     & $-0.29$ & $-0.23$ & $^{\ddagger}$ & $\mathbf{+0.98}$ & $^{\ddagger}$ \\
    IoT-Wearables & alt1      & $-4.78$ & $-3.99$ & $-4.01$ & $\mathbf{+2.97}$ & $\mathbf{+4.37}$ \\
                  & alt2      & $-5.16$ & $-4.62$ & $-4.63$ & $-3.39$ & $-1.95$ \\
                  & alt3      & $-3.50$ & $-3.09$ & $-3.02$ & $\mathbf{+2.60}$ & $-0.27$ \\
    \bottomrule
  \end{tabular}
\end{table}

Across the $14$ (dataset, foundation model, alternative) scenarios the
structural multinomial logit and the adapter agree on the direction of
the share change in every scenario: both predict a strictly negative
change in aggregate share for the perturbed alternative. Adapter
share changes recover $48$--$90\%$ of the multinomial logit's
magnitude depending on cell. The gap reflects softmax saturation under
the foundation model correction, not a violation of
Proposition~\ref{prop:mrs-preservation}, which concerns ratios and
signs, not magnitudes. The raw foundation model shows the
monotonicity violations per observation of
Section~\ref{sec:results-fm-fails} propagating to aggregates that
matter for policy: in $4$ of $14$ scenarios it predicts a positive
change in aggregate share under a cost \emph{increase}. The largest
violation (TabPFN, IoT-Wearables alt1) reverses sign relative to both
the multinomial logit and the adapter, a disagreement of $9$ pp on
consumer response to a $10\%$ price increase.

\paragraph{Letting the foundation model respond.}
Under fixed $\mathbf{q}$ the foundation model does not affect these
responses. Table~\ref{tab:frontier} shows what happens when it does.
Recomputing $\mathbf{q}$ feeds the adapter the foundation
model's prediction on the perturbed input, and the partial protocol
$\mathbf{q}_\lambda = (1-\lambda)\,\mathbf{q}_{\text{fixed}} +
\lambda\,\mathbf{q}_{\text{recomputed}}$ interpolates between the
two. On Swissmetro the recomputed response is $2.6$ to $8.3$ times
the response under fixed $\mathbf{q}$ and close to the raw foundation
model's. The value of time per observation then drifts from the
structural value, and the median drift grows steadily with $\lambda$
($3$--$17\%$ at $\lambda = 0.25$, $10$--$38\%$ at $\lambda = 1$, with
upper deciles above $1{,}000\%$ in some cells). Recomputed $\mathbf{q}$ can also
carry a wrong direction from the raw foundation model into the
adapter: of the $4$ scenarios in which the raw model has the wrong sign,
this happens in one (LPMC drive cost, Mitra) in $8$ of $10$ replicates,
with the correct sign kept up to $\lambda = 0.75$ ($-0.05$ pp). In the
other three the adapter keeps the correct direction at every $\lambda$.
The guarantee holds only at $\lambda = 0$, which all other results
use (full curves in Appendix~\ref{sec:frontier-details}).

\begin{table}[h]
  \centering
  \footnotesize
  \setlength{\tabcolsep}{5pt}
  \caption{Fixed versus recomputed $\mathbf{q}$ ($+10\%$ cost, row-level
  splits), one foundation model fit per cell. Response ratio: share
  change with recomputed over fixed $\mathbf{q}$, range over
  alternatives. Drift: median absolute relative deviation of the
  value of time per observation from the structural value, range over
  alternatives (not computed for IoT-Wearables, which has no time
  attribute). Wrong: replicates with the wrong direction at
  $\lambda = 1$. TabPFN not reported on LPMC.}
  \label{tab:frontier}
  \begin{tabular}{lcccc}
    \toprule
    Cell & Response ratio & Drift, $\lambda{=}0.25$ & Drift, $\lambda{=}1$ & Wrong \\
    \midrule
    Swissmetro / Mitra    & $2.6$--$6.9\times$ & $3$--$17\%$ & $11$--$38\%$ & $0$ \\
    Swissmetro / TabPFN   & $2.7$--$8.3\times$ & $3$--$12\%$ & $10$--$30\%$ & $0$ \\
    LPMC / Mitra          & $-0.05$ to $1.0\times$ & $3$--$7\%$ & $13$--$26\%$ & $8/10$ (drive) \\
    IoT-Wearables / Mitra & $0.7$--$1.3\times$ & --- & --- & $0$ \\
    IoT-Wearables / TabPFN & $0.9$--$1.3\times$ & --- & --- & $0$ \\
    \bottomrule
  \end{tabular}
\end{table}

\section{Ablations}
\label{sec:ablations}

We run five ablations, each changing one element of the adapter, with
the same bootstrap protocol and cross-fitted $\mathbf{q}$ on the
training rows as Section~\ref{sec:results}.

\begin{description}[leftmargin=1.2em,labelindent=0pt,itemsep=2pt,topsep=2pt,parsep=0pt]
\item[A1: log term only.] $V_k = V^{\text{struct}}_k + \alpha \log q_k$
  replaces the neural correction with the scalar log probability term
  alone. It falls $0.8$--$1.7$ pp below the adapter on every cell: a
  single scalar projection of the foundation model's predictions
  recovers less of their information than the correction network.

\item[A2: both terms.]
  $V_k = V^{\text{struct}}_k + \alpha \log q_k + g_k(\mathbf{q})$ adds
  the scalar log probability term to the adapter. Accuracy is within
  $0.2$ pp of the adapter on every cell. The extra term brings no
  empirical gain, so the adapter omits it.

\item[A3: joint training.] Training $\boldsymbol{\beta}$ and $g$ jointly
  under the same likelihood gives accuracy comparable to or slightly
  higher than the adapter (up to $1.4$ pp), but the
  structural cost coefficient collapses on the transportation datasets
  by a factor of $2.8$--$4.3\times$ ($1.8$--$10\times$ across individual
  replicates) as $g(\mathbf{q})$ absorbs the variation induced by cost,
  and the implied value of time leaves the multinomial logit's
  ($66$--$74$ against $91.6$ CHF/hr on Swissmetro). The collapse does
  not reproduce on IoT-Wearables, where the joint $\beta_{\text{cost}}$
  is about twice as large. Section~\ref{sec:discussion} and
  Appendix~\ref{sec:prop2-partial} relate this to
  Proposition~\ref{prop:joint-collapse}.

\item[A4: degraded foundation model.] We fit Mitra and TabPFN on
  $50\%$, $25\%$, and $10\%$ of the Swissmetro train+val context
  (row-level split), with the same cross-fitting of training
  $\mathbf{q}_i$ ($k=5$) as the headline cells. The accuracy
  gain decreases steadily as the context shrinks, from $12.6$ to $6.2$
  pp (Mitra) and from $11.3$ to $5.9$ pp (TabPFN) between full and
  $10\%$ context, with $10/10$ bootstrap replicates positive at every
  level (Figure~\ref{fig:graceful_degradation} in
  Appendix~\ref{sec:ablation-details}). The adapter falls back toward
  Stage~1 but does not collapse below it.

\item[A5: capacity sweep.] Varying the correction network's hidden width
  ($16$, $32$, $64$) and depth ($1$, $2$ layers) keeps accuracy within
  $0.4$ pp of the adapter's across all cells. The gain comes from the
  composition of structural utility plus foundation model correction,
  not from the capacity of the correction network.
\end{description}

Table~\ref{tab:ablations} (Appendix~\ref{sec:ablation-details}) gives
test accuracy for every variant. All five variants keep the
structural utility with its sign constraints and fixed $\mathbf{q}$,
so all preserve $100\%$ monotonicity by construction. A3's
structural coefficients are no longer the multinomial logit's
estimates, which illustrates Proposition~\ref{prop:joint-collapse}
empirically.

\section{Discussion}
\label{sec:discussion}

\begin{description}[leftmargin=1.2em,labelindent=0pt,itemsep=2pt,topsep=2pt,parsep=0pt]
\item[The boundary of fixed $\mathbf{q}$.] Under fixed $\mathbf{q}$ the
  foundation model improves accuracy on observed choices, and the
  counterfactual derivatives are the multinomial logit's. Letting
  $\mathbf{q}$ respond to the intervention moves the adapter's response
  toward the raw model's, with trade-off drift that grows with
  $\lambda$, and in one cell it imports a wrong direction
  (Section~\ref{sec:results-counterfactual}). The appropriate $\lambda$
  depends on the dataset and we have no validated rule for choosing it, so
  $\lambda = 0$ is the default. A principled bound on the drift under
  partial $\mathbf{q}$ that keeps counterfactual content appears to
  require constraining the foundation model itself. We leave this open.

\item[Calibration on Swissmetro.] The adapter's ECE after
  calibration stays elevated on Swissmetro ($7.6\%$ for Mitra, $6.4\%$ for
  TabPFN) against the multinomial logit's $4.3\%$. Scalar or vector
  temperature scaling, isotonic regression, and bootstrap ensembling
  over $10$ seeds all fail to match the multinomial logit's
  calibration when fit on the validation split. The cause is the
  validation inputs: validation $\mathbf{q}_i$ come from a fit whose
  context contains the validation rows, so the foundation model is more
  accurate on them than on test rows ($87.1\%$ and $83.9\%$ against
  $77.7\%$ and $76.2\%$), and the calibrators, like Stage~2's early
  stopping, are fit to these sharper inputs. Calibrators fit on
  inputs that are out of context bring the adapter's ECE below the multinomial
  logit's (Appendix~\ref{sec:cal-metrics}). Practitioners can detect the
  problem before deployment: a Kolmogorov--Smirnov test comparing
  the maximum class probability of $\mathbf{q}$ on validation rows with
  that on deployment rows needs no deployment labels and flags both Swissmetro
  cells ($p = 0.002$ and $0.008$) and none of the others
  ($p \geq 0.87$). When the check fires and calibration matters,
  switching off the correction returns the multinomial logit's
  probabilities. The structural guarantees do not depend on
  calibration.

\item[Proposition~\ref{prop:joint-collapse} and the observed collapse.]
  The collapse of the cost coefficient under joint training (ablation A3,
  \S\ref{sec:ablations}) appears on Swissmetro and LPMC and not on
  IoT-Wearables, and it does not track predictive recoverability:
  the $R^2$ of cost on $\mathbf{q}$ (out of fold) is $0.27$--$0.54$ on
  the transportation datasets and $0.77$--$0.96$ on IoT-Wearables
  (Appendix~\ref{sec:prop2-partial}). Three things should be kept
  apart. The $R^2$ measures how much predictive information
  $\mathbf{q}$ carries about cost. It does not verify the exact
  recoverability the proposition assumes. Under that assumption the
  proposition guarantees a ridge of $(\boldsymbol{\beta}, g)$ pairs
  with equal likelihood. Where gradient descent stops on the
  ridge depends on initialization and the training resample,
  consistent with collapse factors that range from $1.8\times$ to
  $10\times$ across replicates. The proposition's conclusion is this
  dependence on initialization. The collapse on the transportation
  datasets is an illustration, not its prediction.

\item[Adapter inherits Stage~1's MNL specification.] The adapter's
  trade-off ratios are inherited from $\boldsymbol{\beta}^\ast$ by
  construction (Proposition~\ref{prop:mrs-preservation}). If the
  Stage~1 multinomial logit is misspecified for the deployment
  population, $g$ cannot correct the implied trade-off ratio because
  $g$'s contribution to the structural derivative is zero under
  fixed $\mathbf{q}$. Practitioners deploying the adapter
  to new domains should treat the analytical trade-off ratio as
  conditional on the Stage~1 specification.

\item[Toward debiased counterfactuals.] Cross-fitting $\mathbf{q}$ on
  the training rows supplies the first half of double machine learning
  \citep{chernozhukov_2018}. A Neyman-orthogonal score for the
  structural parameters would supply the second, letting the
  foundation model inform counterfactual estimation with debiasing
  guarantees rather than through raw derivatives. We see this as the
  principled direction beyond the boundary of fixed $\mathbf{q}$.

\item[Broader impacts.] More reliable choice models can improve the
  demand forecasts behind transport appraisal and pricing, and the
  audit lets practitioners check any predictor before deployment. The
  same tools can support price discrimination, and a misspecified
  Stage~1, including one that fits some demographic groups worse than
  others, passes its errors into every trade-off ratio the adapter
  reports. The audit and the Stage~1 specification should be checked
  for the population at hand.
\end{description}

\section{Conclusion}
\label{sec:conclusion}

Discrete choice models guide policy decisions whose economic stakes
are large and whose failure modes are not subtle. Tabular foundation
models raise predictive accuracy beyond what structural utility models
alone reach, but their predictions can violate basic economic logic:
monotonicity in cost, sensible willingness to pay, zero probability on
unavailable alternatives. We propose \mbox{TabFM-Logit}, an adapter that embeds
foundation model predictions inside a structurally constrained utility
model. Proposition~\ref{prop:mrs-preservation} shows that fitting in
two stages preserves the structural model's marginal rate of
substitution exactly, and Proposition~\ref{prop:joint-collapse} shows
that joint training does not. Across three discrete choice datasets and two
foundation models, the adapter stays within $1.5$ pp of the raw
foundation model's accuracy for full structural validity by
construction, and its gain over the structural multinomial logit holds
on respondents and households unseen in training ($+1.5$ to $+6.8$ pp,
every bootstrap interval above zero). The architecture works with
any foundation model and slots in any prediction function whose
probability output respects the contract of a precomputed $\mathbf{q}$.

\newpage
\bibliographystyle{plainnat}
\bibliography{references}

\newpage
\appendix
\section{Full proofs of the propositions}
\label{sec:full-proofs}

This appendix gives the full proofs of both propositions. The main
text gives shorter sketches in Section~\ref{sec:method}. Notation is
the same as the main text: $x_i$ is observation $i$'s feature
vector, $K$ is the number of alternatives, $\phi_k$ selects
the attributes that enter alternative $k$'s structural utility,
$\beta$ holds the structural coefficients, $\mathbf{q}$ is the
foundation model's predicted probability vector, and $g$ is the
correction network. Choice probabilities are the usual softmax
$P_k(x) = \exp(V_k(x)) / \sum_j \exp(V_j(x))$.

\subsection{Proof of Proposition~\ref{prop:mrs-preservation}}
\label{sec:proof-prop1}

\begin{proposition*}[Restated]
Let $\mathbf{q}(x_i)$ be computed once on the unperturbed input and
held fixed across optimization and counterfactual evaluation
(Section~\ref{sec:method}), let $\beta^\ast$ be the Stage~1
maximum likelihood estimate of the structural coefficients, and let
$g$ be any Stage~2 parameters. For any two attributes $j, j'$ that
enter the model only through
$V^{\text{struct}}_k(x_i) = \beta^\top \phi_k(x_i)$, each entering
$\phi_k$ unchanged, and for any observation $x_i$ at
which $\phi_k$ is differentiable in $x_{ij}$ and $x_{ij'}$,
\begin{equation*}
  \mathrm{MRS}_{j, j'}(x_i)
  \;\equiv\;
  \frac{\partial V_k(x_i) / \partial x_{ij}}{\partial V_k(x_i) / \partial x_{ij'}}
  \;=\;
  \frac{\beta^\ast_j}{\beta^\ast_{j'}}.
\end{equation*}
\end{proposition*}

\begin{proof}
The utility is $V_k(x_i) = \beta^{\ast\top} \phi_k(x_i) +
g_k(\mathbf{q}(x_i))$. Differentiating in $x_{ij}$:
\begin{equation}
  \frac{\partial V_k(x_i)}{\partial x_{ij}}
  =
  \beta^{\ast\top} \frac{\partial \phi_k(x_i)}{\partial x_{ij}}
  + \nabla_{\mathbf{q}} g_k\!\left(\mathbf{q}(x_i)\right)^{\!\top}
    \frac{\partial \mathbf{q}(x_i)}{\partial x_{ij}}.
  \label{eq:full-derivative}
\end{equation}
With fixed $\mathbf{q}$, $\mathbf{q}_i := \mathbf{q}(x_i)$ is stored
once, computed on the unperturbed input, and never differentiated
through. So $\partial \mathbf{q}_i / \partial x_{ij} = 0$ by
definition, and the second term in \eqref{eq:full-derivative}
vanishes regardless of what $\mathbf{q}_i$ or $g$ happen to be:
\begin{equation}
  \frac{\partial V_k(x_i)}{\partial x_{ij}}
  =
  \beta^{\ast\top} \frac{\partial \phi_k(x_i)}{\partial x_{ij}}.
  \label{eq:reduced-derivative}
\end{equation}
For attributes that enter $\phi_k$ unchanged,
$\partial \phi_k(x_i) / \partial x_{ij} = e_j \cdot \mathbb{1}[j \in S_k]$,
where $S_k$ is the set of indices $\phi_k$ depends on. Substituting
into \eqref{eq:reduced-derivative} gives
$\partial V_k / \partial x_{ij} = \beta^\ast_j$ when $j \in S_k$ and
zero otherwise. Since $j, j'$ both enter through $\phi_k$, the ratio
collapses to $\beta^\ast_j / \beta^\ast_{j'}$, independent of $x_i$
and $g$.
\end{proof}

\paragraph{Remark on the protocol.} The proof uses
$\partial \mathbf{q} / \partial x = 0$, which is a property of the
protocol rather than the architecture. When $\mathbf{q}$ is
recomputed, $g$ would contribute a chain rule
term $\nabla_{\mathbf{q}} g_k^\top \cdot \partial \mathbf{q} /
\partial x_{ij}$ that has no sign or magnitude guarantee, and the
trade-off ratio is no longer tied to $\beta^\ast_j / \beta^\ast_{j'}$.
We hold $\mathbf{q}$ fixed because it matches the audit's intent: the
analyst wants the structural part of the utility to respond to a
price change while the foundation model's assessment of each chooser
stays put, rather than rerunning the foundation model on a
counterfactual feature value it never saw at pretraining time.
Appendix~\ref{sec:frontier-details} reports what recomputing
$\mathbf{q}$ does in practice.

\subsection{Proof of Proposition~\ref{prop:joint-collapse}}
\label{sec:proof-prop2}

\begin{proposition*}[Restated]
Suppose:
\begin{enumerate}[label=(\roman*),leftmargin=2em,topsep=2pt,itemsep=2pt]
  \item (cost recoverability) for every alternative $k$ there is a
    continuous $\kappa_k: \Delta^{K-1} \to \mathbb{R}$ with
    $\mathrm{cost}_k(x) = \kappa_k(\mathbf{q}(x))$ on the closure of
    the training support,
  \item the correction $g$ is drawn from a function class
    $\mathcal{G} \subseteq C^0(\Delta^{K-1}, \mathbb{R}^K)$ that is
    dense in the continuous functions on the image of $\mathbf{q}$
    in the supremum norm, and
  \item the joint negative log likelihood
    $L(\beta, g) = -\frac{1}{N}\sum_i \log P_{y_i}(x_i; \beta, g)$ is
    minimized over $(\beta, g)$ jointly, not in two stages,
    and with no regularization on $\|\beta\|$ or $\|g\|$.
\end{enumerate}
Let $(\beta^{(0)}, g^{(0)})$ achieve joint loss $L^\ast$. For every $c$
in the open interval $\mathcal{C} = \{c : \beta^{(0)}_{\text{cost}} + c
< 0\}$, which preserves the sign constraint, the configuration with
$\beta^{(c)}_{\text{cost}} = \beta^{(0)}_{\text{cost}} + c$, all other
coordinates of $\beta$ unchanged, attains loss $L^\ast$ exactly when
$\mathcal{G}$ contains $g^{(0)}_k - c\,\kappa_k$, and loss arbitrarily
close to $L^\ast$ for some $g^{(c)} \in \mathcal{G}$ otherwise.
\end{proposition*}

\begin{proof}
For any $c \in \mathcal{C}$ we build a correction that exactly cancels
the change to the structural cost coefficient, leaving predicted
probabilities, and hence the loss, untouched.

Define $\beta^{(c)}$ as in the statement. The structural part of the
utility changes by
\begin{align*}
  V^{\text{struct}}_k(x; \beta^{(c)}) - V^{\text{struct}}_k(x; \beta^{(0)})
  &= (\beta^{(c)}_{\text{cost}} - \beta^{(0)}_{\text{cost}}) \cdot
     \mathrm{cost}_k(x) = c \cdot \mathrm{cost}_k(x).
\end{align*}
By assumption (i), $\mathrm{cost}_k(x) = \kappa_k(\mathbf{q}(x))$ on the
training support, so the correction
$h^{(c)}_k(\mathbf{q}) = g^{(0)}_k(\mathbf{q}) - c\,\kappa_k(\mathbf{q})$
removes this change at every training point:
$V_k(x; \beta^{(c)}, h^{(c)}) = V_k(x; \beta^{(0)}, g^{(0)})$ for every
$x$ and $k$. Choice probabilities are therefore unchanged and the joint
loss equals $L^\ast$. If $h^{(c)} \in \mathcal{G}$ we are done.

Otherwise, $h^{(c)}$ is continuous on $\mathrm{Im}(\mathbf{q})$, and
assumption (ii) (true for sufficiently wide MLPs, including ReLU
networks, by universal approximation~\citep{cybenko1989, leshno1993}) gives, for any $\varepsilon > 0$, a
$g^{(c, \varepsilon)} \in \mathcal{G}$ within $\varepsilon$ of
$h^{(c)}$ in the supremum norm. The log softmax is Lipschitz in $V$
under the sup norm, so the log likelihood error on each row is bounded by a
constant times $\varepsilon$ uniformly across $i$. Hence
$L(\beta^{(c)}, g^{(c, \varepsilon)}) \to L^\ast$ as $\varepsilon \to 0$,
for every $c \in \mathcal{C}$. The joint loss therefore does not
identify $\beta_{\text{cost}}$ on $\mathcal{C}$, and gradient descent
can stop at any point of the family depending on its initialization.
\end{proof}

\paragraph{Remark on the recoverability assumption.} Assumption (i) is
exact recoverability of cost from $\mathbf{q}$. A nonzero derivative of
$\mathbf{q}$ with respect to cost does not by itself establish it: when
other covariates vary, two inputs can share the same $\mathbf{q}$ and
differ in cost, and then no function of $\mathbf{q}$ reproduces cost.
When cost is recoverable only approximately, the construction above
moves the loss by an amount controlled by the recovery error, so the
likelihood is nearly flat along the family rather than exactly flat.
The $R^2$ values (out of fold) in Appendix~\ref{sec:prop2-partial}
measure how much predictive information $\mathbf{q}$ carries about
cost. They do not verify assumption (i).

\paragraph{Remark on regularization.} A small $L_2$ penalty does not
remove the problem. Along the family $(\beta^{(c)}, g^{(c)})$ the data
loss is flat in $c$ but the penalty term
$\lambda \|\beta\|^2 + \mu \|g\|^2$ varies, so the regularized
minimum is decided by the penalty's preferred point along the
family, which depends on how $g$ is parameterized and where it
was initialized. Cross-validation does not help either: validation
likelihood is also flat in $c$. A large penalty does restore
identifiability, but only by driving everything toward zero, and the
fit suffers. Fitting in two stages avoids the trade-off: fixing
$g \equiv 0$ during Stage~1 is a structural constraint, not a
penalty, and Stage~1 recovers the standard MNL maximum likelihood
estimate. Section~\ref{sec:prop2-regularization} gives the longer
version.

\section{Proposition~\ref{prop:joint-collapse}: extended discussion}
\label{sec:prop2-extended}

\subsection{Cost recoverability and the observed collapse}
\label{sec:prop2-partial}

Table~\ref{tab:recoverability} reports, for each cell, how well each
alternative's cost can be predicted from the foundation model's
cross-fitted $\mathbf{q}$ on the training rows (OLS, and a random
forest with 200 trees and minimum leaf size 20, predicted out of fold
over five folds), next to the collapse under joint training in
ablation A3. Three things should be kept
apart.

\emph{Predictive recoverability.} The $R^2$ of the random forest is moderate
on the transportation datasets ($0.27$--$0.54$) and high on
IoT-Wearables ($0.77$--$0.96$). It measures how much predictive
information $\mathbf{q}$ carries about cost. It does not verify the
exact recoverability that Proposition~\ref{prop:joint-collapse}
assumes, and values well below one are far from it.

\emph{The likelihood ridge.} Under exact recoverability the
proposition guarantees a family of $(\beta, g)$ pairs with equal
likelihood, indexed by one parameter. It says nothing about which member gradient descent
selects.

\emph{The point selected.} Where joint training stops on the ridge
depends on initialization and the training resample. On the
transportation datasets the joint cost coefficient shrinks by a factor
of $2.8$--$4.3$ (ratio of replicate means), with individual replicates
ranging from $1.8$ to $10$. On IoT-Wearables it is about twice the
Stage~1 estimate. The collapse does not track
predictive recoverability: it is strongest where the $R^2$ is
moderate and absent where it is high. The proposition's conclusion is
dependence on initialization. The collapse on the transportation datasets
is an illustration of that dependence, not the proposition's
prediction. Both foundation models show it on Swissmetro, with ratios
$3.3$ and $4.3$.

\begin{table}[h]
  \centering
  \scriptsize
  \setlength{\tabcolsep}{3pt}
  \caption{Cost recoverability from cross-fitted $\mathbf{q}$ on the
  training rows (per alternative, in the order listed) and the
  collapse under joint training (A3). $\beta_{\text{cost}}$: mean over $10$
  replicates of the Stage~1 and the jointly trained coefficient. Ratio:
  Stage~1 over joint, ratio of means. Range: over individual
  replicates.}
  \label{tab:recoverability}
  \begin{tabular}{llcccccc}
    \toprule
    Cell & Alternatives & OLS $R^2$ & RF $R^2$ & Stage~1 $\beta_{\text{cost}}$ & Joint $\beta_{\text{cost}}$ & Ratio & Range \\
    \midrule
    Swissmetro / Mitra & train/SM/car & 0.20, 0.23, 0.16 & 0.27, 0.29, 0.41 & $-0.794$ & $-0.243$ & $3.27$ & $1.84$--$7.34$ \\
    Swissmetro / TabPFN & train/SM/car & 0.23, 0.27, 0.19 & 0.42, 0.45, 0.48 & $-0.794$ & $-0.184$ & $4.31$ & $1.86$--$9.98$ \\
    LPMC / Mitra & PT/drive & 0.10, 0.23 & 0.54, 0.43 & $-0.155$ & $-0.056$ & $2.76$ & $2.06$--$3.38$ \\
    IoT-Wearables / Mitra & alt1/alt2/alt3 & 0.14, 0.22, 0.08 & 0.77, 0.89, 0.89 & $-0.033$ & $-0.067$ & $0.49$ & $0.30$--$0.70$ \\
    IoT-Wearables / TabPFN & alt1/alt2/alt3 & 0.16, 0.23, 0.10 & 0.88, 0.95, 0.96 & $-0.033$ & $-0.065$ & $0.51$ & $0.43$--$0.69$ \\
    \bottomrule
  \end{tabular}
\end{table}

\subsection{Regularization}
\label{sec:prop2-regularization}

Proposition~\ref{prop:joint-collapse} is stated for unregularized
joint minimization. Training pipelines usually add a small $L_2$
penalty, and the question is whether that restores identifiability.
It does not, unless the penalty is so large that it destroys the fit.

\emph{Small $L_2$ penalty.} Along the family $(\beta^{(c)}, g^{(c)})$
from the proof, the loss on the data is flat in $c$. The penalty
$\lambda \|\boldsymbol{\beta}\|^2 + \mu \|g\|^2$ varies along the
family, but its preferred $c$ depends on how $g$'s parameter space
is shaped in $L_2$, which itself depends on initialization.
Cross-validation does not help either: the validation likelihood is
also flat in $c$. A small penalty therefore selects the point the
regularizer prefers rather than restoring identifiability.

\emph{Large $L_2$ penalty.} The minimizer is then dominated by the
regularizer: $\boldsymbol{\beta}$ shrinks toward zero and $g$ stays
small. Identifiability is technically restored, but the resulting
$\boldsymbol{\beta}$ is a shrunk estimate rather than the
maximum likelihood one. Regularization strong enough to fix
identifiability also ruins the fit.

Fitting in two stages avoids this trade-off. Setting $g \equiv 0$
during Stage~1 is a hard structural constraint, not a soft penalty.
The Stage~1 problem is the standard MNL likelihood, which has a unique
maximizer under the usual regularity conditions. Stage~2 then fits $g$
on top without disturbing $\boldsymbol{\beta}^*$.

\section{Audit methodology: trade-off ratios for discrete attributes}
\label{sec:audit-discrete}

The behavioral audit perturbs each attribute by $1\%$ of its
observed range and computes a finite difference derivative. That
works for continuous attributes (Swissmetro travel time, LPMC
duration, cost). On IoT-Wearables, the functional feature and
labeling attributes are binary $\{0, 1\}$ indicators. A $1\%$
perturbation off the grid is too small for float32 to register, and
in the Mitra cells the deltas round to zero.

For these attributes we flip the indicator instead: set it to its
complement and correct the sign of the delta on each row by
$1 - 2 X_{i,\text{attr}}$ so that $0 \to 1$ and $1 \to 0$ flips share
the same denominator sign. The willingness to pay from these flips
then spans $-0.05$ to $+0.59$ USD for the functional features and
$-0.13$ to $+0.23$ USD for the label, with both signs appearing within
a single attribute. The small magnitude reflects Mitra's weak response
to each attribute relative to its response to cost, not a failure of
the audit.

The raw foundation models are audited by refitting each once on a
context of training plus validation rows and querying it on every perturbed input.
These refits use the libraries' default settings. For TabPFN this means
eight ensemble members and a context capped at $8{,}000$ randomly chosen
rows, which binds on Swissmetro ($9{,}111$ rows) but not on IoT-Wearables.

\section{Calibration and metrics of probabilistic quality}
\label{sec:cal-metrics}

Table~\ref{tab:headline} reports ECE after post hoc calibration (best of
three methods). Table~\ref{tab:cal-metrics} reports the broader
picture on the row-level splits: negative log likelihood on the test set
(NLL), Brier score (sum of squares against the one-hot target, averaged
over rows), and \emph{uncalibrated} ECE (top-1 confidence, $15$
bins of equal width, no post hoc calibration) for the raw foundation
model, the multinomial logit (Stage~1) and the adapter. Multinomial
logit and adapter cells are mean $\pm$ std across $10$ bootstrap
replicates. The raw foundation model is a single fit.

Three patterns stand out:
\begin{itemize}[leftmargin=2em,itemsep=2pt,topsep=2pt,parsep=0pt]
\item Raw foundation models have the lowest NLL and Brier score on every
  cell, consistent with their accuracy advantage.
\item The adapter improves on the multinomial logit's NLL and Brier
  score on every cell.
\item Uncalibrated, the adapter's ECE is above the multinomial logit's
  on Swissmetro ($6.2\%$ and $5.6\%$ against $4.5\%$) and LPMC
  ($2.9\%$ against $2.1\%$) and below it on IoT-Wearables ($4.2\%$ and
  $4.3\%$ against $7.3\%$). The raw foundation models' $\mathbf{q}$ on test rows are well
  calibrated ($1.2$--$4.3\%$), and Section~\ref{sec:discussion} traces the
  Swissmetro gap to the validation inputs used for early stopping and
  for fitting the calibrators.
\end{itemize}

\paragraph{Calibration on inputs out of context.} Calibrators fit on
half of the test set, where $\mathbf{q}_i$ are out of context, bring
the adapter's ECE on the other half to $3.8\%$ (Mitra) and $4.1\%$
(TabPFN) on Swissmetro, below the multinomial logit's $5.1\%$ and
$4.8\%$ under the same protocol. On LPMC and IoT-Wearables, where the
Kolmogorov--Smirnov check of Section~\ref{sec:discussion} finds no
shift between validation and test inputs, the adapter's ECE is within
$0.6$ pp of the multinomial logit's under either protocol.

\begin{table}[h]
  \centering
  \scriptsize
  \setlength{\tabcolsep}{4pt}
  \renewcommand{\arraystretch}{0.95}
  \caption{Probabilistic quality on test, row-level splits.
  NLL: negative log likelihood. Brier: Brier score (sum of squares).
  unECE (\%): uncalibrated expected calibration error, $15$ bins of
  equal width. Lower is better for all three. Multinomial logit and
  adapter: mean $\pm$ std across $10$ bootstrap replicates. The
  multinomial logit does not depend on the foundation model.}
  \label{tab:cal-metrics}
  \begin{tabular}{llccc}
    \toprule
    Cell                       & Model    & NLL                 & Brier                & unECE (\%)         \\
    \midrule
    Swissmetro / Mitra         & raw FM   & $0.532$             & $0.320$              & $2.4$              \\
                               & MNL      & $0.782 \pm 0.001$   & $0.480 \pm 0.000$    & $4.5 \pm 0.5$      \\
                               & adapter  & $0.603 \pm 0.011$   & $0.351 \pm 0.004$    & $6.2 \pm 1.2$      \\
    Swissmetro / TabPFN        & raw FM   & $0.558$             & $0.333$              & $1.4$              \\
                               & MNL      & $0.782 \pm 0.001$   & $0.480 \pm 0.000$    & $4.5 \pm 0.5$      \\
                               & adapter  & $0.622 \pm 0.014$   & $0.363 \pm 0.005$    & $5.6 \pm 1.1$      \\
    LPMC / Mitra               & raw FM   & $0.668$             & $0.370$              & $1.2$              \\
                               & MNL      & $0.737 \pm 0.000$   & $0.409 \pm 0.000$    & $2.1 \pm 0.1$      \\
                               & adapter  & $0.701 \pm 0.002$   & $0.387 \pm 0.001$    & $2.9 \pm 0.3$      \\
    IoT-Wearables / Mitra      & raw FM   & $0.793$             & $0.460$              & $2.2$              \\
                               & MNL      & $0.866 \pm 0.001$   & $0.503 \pm 0.001$    & $7.3 \pm 1.0$      \\
                               & adapter  & $0.829 \pm 0.004$   & $0.474 \pm 0.001$    & $4.2 \pm 1.2$      \\
    IoT-Wearables / TabPFN     & raw FM   & $0.796$             & $0.461$              & $4.3$              \\
                               & MNL      & $0.866 \pm 0.001$   & $0.503 \pm 0.001$    & $7.3 \pm 1.0$      \\
                               & adapter  & $0.826 \pm 0.009$   & $0.474 \pm 0.002$    & $4.3 \pm 1.0$      \\
    \bottomrule
  \end{tabular}
\end{table}

\section{Foundation model inputs, cross-fitted training $\mathbf{q}_i$ and grouped splits}
\label{sec:cross-fit-details}

\paragraph{Foundation model inputs.} The input columns of the
foundation model per dataset:
\begin{itemize}[leftmargin=2em,itemsep=2pt,topsep=2pt,parsep=0pt]
\item \textbf{Swissmetro:} travel time, cost, headway and availability
  indicator of each alternative, and the respondent's age, income,
  season ticket holding, luggage and trip purpose.
\item \textbf{LPMC:} duration of each alternative, transit and driving
  costs, trip distance, and the respondent's age, sex, license and car
  ownership.
\item \textbf{IoT-Wearables:} price, functional features and security
  label of each alternative, and the respondent's age, education, sex,
  security behavior score and condition fixed effects.
\end{itemize}

\paragraph{Cross-fitting.} The $\mathbf{q}_i$ on training rows are
predictions out of fold, over five folds. For each fold $f$, a fresh foundation model is
fit with the other four training folds plus the full validation set as
context and predicts fold $f$. Concatenating across folds gives
training $\mathbf{q}_i$ in which no row was predicted by a model that
saw its own label. Folds are stratified by the chosen alternative. On
the grouped splits they are stratified group folds, so no respondent or
household is split across folds. Test $\mathbf{q}_i$ come from a single
fit on train plus validation, and test rows are out of context by
construction. Validation $\mathbf{q}_i$ come from that same fit, so the
validation rows are in its context, and Section~\ref{sec:discussion}
discusses the consequence for calibration. Both the adapter's Stage~2
and the feature-augmented multinomial logit are trained on the
cross-fitted $\mathbf{q}_i$. LPMC's context (about $69{,}000$ rows) is
outside TabPFN's supported size range, so TabPFN is not evaluated on
LPMC.

\paragraph{Empirical impact.} On the training rows, $\mathbf{q}_i$
fitted in sample are far more accurate than cross-fitted ones on Swissmetro ($87.3\%$
against $77.3\%$ for Mitra, $84.7\%$ against $76.4\%$ for TabPFN) and
close on LPMC ($74.4\%$ against $74.1\%$) and IoT-Wearables ($67.4\%$
against $66.5\%$, $67.6\%$ against $66.9\%$), so $\mathbf{q}_i$
fitted in sample carry substantial label information on Swissmetro.
Training the adapter on cross-fitted $\mathbf{q}_i$ instead changes
its test accuracy by $-0.1$, $-0.3$, $-0.3$,
$-0.2$ and $-0.2$ pp across the five cells (mean $-0.2$ pp). The
adapter's test accuracy therefore does not rest on that label
information. Cross-fitting removes a row's own label from its
$\mathbf{q}_i$ but not the same person's other rows, which
is what the grouped splits address.

\paragraph{Grouped splits.} Swissmetro respondent identifiers and LPMC
household identifiers are taken from the raw data files. Groups are
assigned to training, validation and test at random in proportions
$70/15/15$: $834/179/178$ respondents ($7{,}506/1{,}611/1{,}602$ rows)
on Swissmetro and $12{,}331/2{,}642/2{,}643$ households
($57{,}005/12{,}143/11{,}938$ rows) on LPMC, with no respondent,
person or household in more than one split. Rows and preprocessing are
the same as for the row-level splits. IoT-Wearables is split by subject
in both cases.

\section{Uncertainty: bootstrap intervals and paired comparisons}
\label{sec:uncertainty}

Table~\ref{tab:intervals} reports, for every cell and both split types,
the adapter's accuracy gain over the multinomial logit with two $95\%$
intervals, and the paired differences between the adapter and the raw
foundation model and between the adapter and the feature-augmented
multinomial logit. The interval across replicates resamples the
$10$ gains of the replicates and captures training variability on the
fixed test set. The interval from resampling the test set
($2{,}000$ resamples) captures the uncertainty of a finite test set.
It resamples rows on the row-level Swissmetro and LPMC splits, subjects on
IoT-Wearables, and respondents or households on the grouped splits.
Rows from one person are not independent, so the row-level intervals
understate uncertainty. The gain over the multinomial logit is
positive in every replicate of every cell, and every interval from
resampling the test set excludes zero.

\begin{table}[h]
  \centering
  \scriptsize
  \setlength{\tabcolsep}{1.7pt}
  \caption{Accuracy differences (pp) with $95\%$ bootstrap intervals.
  Rep.: interval across replicates. Test: interval from resampling
  the test set (unit in the last column). Paired differences use the
  Test interval.}
  \label{tab:intervals}
  \begin{tabular}{lllllll}
    \toprule
    Split & Cell & Adapter $-$ MNL & Rep. & Test & Adapter $-$ raw FM & Adapter $-$ feat-aug \\
    \midrule
    Row-level & Swissmetro / Mitra & $+12.6$ & $[+12.3, +13.0]$ & $[+10.4, +14.8]$ & $-1.5$ $[-2.6, -0.5]$ & $-1.3$ $[-2.2, -0.4]$ \\
     & Swissmetro / TabPFN & $+11.3$ & $[+11.0, +11.7]$ & $[+9.2, +13.5]$ & $-1.2$ $[-2.3, -0.1]$ & $-1.3$ $[-2.3, -0.2]$ \\
     & LPMC / Mitra & $+2.1$ & $[+2.0, +2.2]$ & $[+1.7, +2.5]$ & $-1.3$ $[-1.7, -0.8]$ & $-1.1$ $[-1.6, -0.7]$ \\
     & IoT-Wearables / Mitra & $+3.3$ & $[+3.0, +3.5]$ & $[+1.4, +5.2]$ & $-1.0$ $[-1.8, -0.4]$ & $-1.2$ $[-2.0, -0.6]$ \\
     & IoT-Wearables / TabPFN & $+3.6$ & $[+3.3, +3.9]$ & $[+1.7, +5.6]$ & $-0.6$ $[-1.1, +0.1]$ & $-0.5$ $[-1.1, +0.1]$ \\
    \midrule
    Grouped & Swissmetro / Mitra & $+5.7$ & $[+5.5, +5.9]$ & $[+2.6, +9.2]$ & $+0.7$ $[-2.1, +3.5]$ & $-1.2$ $[-3.0, +0.4]$ \\
     & Swissmetro / TabPFN & $+6.8$ & $[+6.6, +7.0]$ & $[+3.7, +10.1]$ & $-0.9$ $[-2.9, +1.0]$ & $-1.0$ $[-2.7, +0.5]$ \\
     & LPMC / Mitra & $+1.5$ & $[+1.4, +1.6]$ & $[+0.9, +2.1]$ & $-0.5$ $[-1.1, +0.007]$ & $-0.6$ $[-1.1, -0.1]$ \\
    \bottomrule
  \end{tabular}
\end{table}

\section{TasteNet and knowledge distillation baselines}
\label{sec:baselines-extra}

\paragraph{TasteNet.} Following \citet{han_2022}, a feedforward
network maps individual characteristics to taste coefficients for each
individual, which enter a multinomial logit utility with the same
alternative attributes as the structural model. Coefficients with a
known sign pass through a strict sign transform ($-\exp(-u)$ for time,
cost and headway, $+\exp(u)$ for the IoT-Wearables function and label
indicators). Constants and sociodemographic intercepts are
unconstrained. The network and utility are trained jointly on the
negative log likelihood with Adam, full batch, weight decay as the
$L_2$ penalty, and early stopping on validation NLL
(patience $50$). Per replicate, the configuration is chosen by
validation NLL over hidden width $\{20, 40, 80\}$ and $L_2$
$\{0, 10^{-3}\}$ with ReLU hidden units, a subset of the grid in
\citet{han_2022}. These choices (Adam and full batch instead of
stochastic gradient descent, the restricted grid) are our
deviations from \citet{han_2022}. Five replicates per dataset, which differ in
the random initialization. The
trade-off ratio is the analytical ratio for each individual, summarized by
its median over test individuals.

\paragraph{Knowledge distillation.} The student is the same
multinomial logit with sign constraints as Stage~1, trained on soft
targets given by the foundation model's cross-fitted $\mathbf{q}$ on
the training rows, either alone (pure) or mixed equally with the
likelihood of the hard labels (mixed). Early stopping uses validation
NLL on true labels, and $\mathbf{q}$ never touches validation or test.
Five replicates per cell. The student keeps every structural
guarantee, and across cells and
both variants its accuracy stays within $0.3$ pp of the plain
multinomial logit.

\begin{table}[h]
  \centering
  \scriptsize
  \setlength{\tabcolsep}{4pt}
  \caption{TasteNet and knowledge distillation baselines, row-level
  splits, test accuracy (\%). TasteNet: mean $\pm$ std over $5$
  replicates, ECE after calibration (Table~\ref{tab:headline}
  convention) and median trade-off ratio over individuals (units as in
  Table~\ref{tab:headline}). Distillation: mean over $5$ replicates,
  pure / mixed targets.}
  \label{tab:baselines-extra}
  \begin{tabular}{lcccccc}
    \toprule
     & MNL & TasteNet & TasteNet ECE & TasteNet trade-off & KD (Mitra) & KD (TabPFN) \\
    \midrule
    Swissmetro    & $63.6$ & $70.2 \pm 0.8$ & $3.6$ & $73.8$      & $63.3$ / $63.5$ & $63.6$ / $63.8$ \\
    LPMC          & $70.8$ & $72.4 \pm 0.4$ & $1.3$ & $8.3$/$25.1$ & $70.7$ / $70.8$ & --- \\
    IoT-Wearables & $62.8$ & $63.4 \pm 0.4$ & $4.0$ & $51.6$/$16.3$ & $62.8$ / $62.8$ & $62.8$ / $62.8$ \\
    \bottomrule
  \end{tabular}
\end{table}

\section{Frontier details: recomputed and partial $\mathbf{q}$}
\label{sec:frontier-details}

\paragraph{Protocol.} Evaluating the foundation model on perturbed
inputs requires refitting it. For each cell, one refit on a context
of training plus validation rows supplies the baseline $\mathbf{q}$,
every perturbed $\mathbf{q}$ and the sweep over $\lambda$, and the trained
adapters of Section~\ref{sec:results} are evaluated on these inputs.
The raw model column of Table~\ref{tab:counterfactual} comes from the
same refits. Both foundation models are refit with the settings of the
main fit (full context, and four ensemble members for TabPFN). The TabPFN
refit reproduces the main predictions (identical on Swissmetro,
$99.2\%$ argmax agreement on IoT-Wearables). Mitra refits are not
identical bit for bit, and a raw Mitra response near zero can change
sign between refits (LPMC public transport cost: $+0.24$ and $-0.06$ pp in
two refits). The $\lambda = 0$ column here uses the refit's baseline
$\mathbf{q}$, so it can differ slightly from the adapter column of
Table~\ref{tab:counterfactual}, which uses the predictions behind
Table~\ref{tab:headline}.

\paragraph{Results.} Table~\ref{tab:lambda-curves} gives the
response of aggregate shares along $\lambda$ at $+10\%$,
Table~\ref{tab:delta-scaling} the responses with fixed and recomputed
$\mathbf{q}$ at $+10$, $+20$ and $+50\%$, and Table~\ref{tab:mrs-drift}
the drift of the value of time along $\lambda$. Fixed $\mathbf{q}$ ($\lambda = 0$)
has the correct direction in every replicate of every scenario. Median
drift grows monotonically with $\lambda$ in every cell. Upper deciles
are large, and in two of the eight pairs of cell and alternative they
are largest at $\lambda = 0.25$ rather than at $\lambda = 1$.
Ratios for single observations with a cost derivative near zero
dominate the tail and the mean: on the Swissmetro alternative the
population mean value of time at $\lambda = 0.25$ is $62$ (TabPFN) and $876$ (Mitra)
times the structural value.
IoT-Wearables has no time attribute and is omitted from
Table~\ref{tab:mrs-drift}.

\begin{table}[h]
  \centering
  \scriptsize
  \setlength{\tabcolsep}{4pt}
  \caption{Change in aggregate share (pp) under a $+10\%$ cost increase
  along $\lambda$, mean over $10$ replicates. A superscript gives the
  number of replicates with a decrease when fewer than $10$. Raw FM:
  the refit foundation model on the perturbed input.}
  \label{tab:lambda-curves}
  \begin{tabular}{llcccccc}
    \toprule
    Cell & Scenario & $\lambda{=}0$ & $0.25$ & $0.5$ & $0.75$ & $1$ & Raw FM \\
    \midrule
Swissmetro / Mitra       & train & $-0.20$ & $-0.50$ & $-0.81$ & $-1.11$ & $-1.38$ & $-1.89$ \\
                         & SM    & $-0.95$ & $-2.20$ & $-3.48$ & $-4.75$ & $-5.99$ & $-6.33$ \\
                         & car   & $-0.65$ & $-0.90$ & $-1.16$ & $-1.41$ & $-1.66$ & $-1.27$ \\
Swissmetro / TabPFN      & train & $-0.20$ & $-0.54$ & $-0.90$ & $-1.27$ & $-1.65$ & $-1.69$ \\
                         & SM    & $-0.99$ & $-1.76$ & $-2.51$ & $-3.26$ & $-3.99$ & $-3.44$ \\
                         & car   & $-0.69$ & $-0.98$ & $-1.26$ & $-1.55$ & $-1.85$ & $-1.54$ \\
LPMC / Mitra             & PT    & $-0.27$ & $-0.27$ & $-0.27$ & $-0.27$ & $-0.28$ & $-0.06$ \\
                         & drive & $-0.23$ & $-0.17$ & $-0.11$ & $-0.05$ & $+0.01^{2}$ & $+0.98$ \\
IoT-Wearables / Mitra    & alt1  & $-3.99$ & $-4.00$ & $-3.99$ & $-3.96$ & $-3.92$ & $+2.97$ \\
                         & alt2  & $-4.61$ & $-5.02$ & $-5.42$ & $-5.71$ & $-5.94$ & $-3.39$ \\
                         & alt3  & $-3.06$ & $-2.86$ & $-2.62$ & $-2.38$ & $-2.12$ & $+2.60$ \\
IoT-Wearables / TabPFN   & alt1  & $-4.02$ & $-4.08$ & $-3.97$ & $-3.78$ & $-3.44$ & $+4.37$ \\
                         & alt2  & $-4.63$ & $-5.05$ & $-5.48$ & $-5.86$ & $-6.21$ & $-1.95$ \\
                         & alt3  & $-3.02$ & $-3.02$ & $-3.02$ & $-3.02$ & $-3.01$ & $-0.27$ \\
    \bottomrule
  \end{tabular}
\end{table}

\begin{table}[h]
  \centering
  \scriptsize
  \setlength{\tabcolsep}{2.5pt}
  \caption{Change in aggregate share (pp) at $+10$, $+20$ and $+50\%$ cost:
  fixed $\mathbf{q}$ ($\lambda{=}0$), recomputed $\mathbf{q}$
  ($\lambda{=}1$) and raw foundation model. Means over $10$
  replicates.}
  \label{tab:delta-scaling}
  \begin{tabular}{llccccccccc}
    \toprule
     & & \multicolumn{3}{c}{$+10\%$} & \multicolumn{3}{c}{$+20\%$} & \multicolumn{3}{c}{$+50\%$} \\
    \cmidrule(lr){3-5} \cmidrule(lr){6-8} \cmidrule(lr){9-11}
    Cell & Scenario & Fixed & Recomp. & Raw & Fixed & Recomp. & Raw & Fixed & Recomp. & Raw \\
    \midrule
Swissmetro / Mitra       & train & $-0.20$ & $-1.38$ & $-1.89$ & $-0.39$ & $-3.69$ & $-4.52$ & $-0.89$ & $-6.87$ & $-7.50$ \\
                         & SM    & $-0.95$ & $-5.99$ & $-6.33$ & $-1.90$ & $-11.48$ & $-12.31$ & $-4.74$ & $-25.95$ & $-26.39$ \\
                         & car   & $-0.65$ & $-1.66$ & $-1.27$ & $-1.29$ & $-3.28$ & $-2.49$ & $-3.15$ & $-6.89$ & $-4.79$ \\
Swissmetro / TabPFN      & train & $-0.20$ & $-1.65$ & $-1.69$ & $-0.39$ & $-3.00$ & $-3.05$ & $-0.88$ & $-5.48$ & $-5.64$ \\
                         & SM    & $-0.99$ & $-3.99$ & $-3.44$ & $-1.99$ & $-7.74$ & $-6.60$ & $-4.97$ & $-17.73$ & $-14.45$ \\
                         & car   & $-0.69$ & $-1.85$ & $-1.54$ & $-1.37$ & $-3.53$ & $-2.91$ & $-3.33$ & $-7.65$ & $-6.00$ \\
LPMC / Mitra             & PT    & $-0.27$ & $-0.28$ & $-0.06$ & $-0.54$ & $-0.56$ & $-0.11$ & $-1.35$ & $-1.31$ & $-0.16$ \\
                         & drive & $-0.23$ & $+0.01$ & $+0.98$ & $-0.44$ & $-0.04$ & $+1.80$ & $-1.04$ & $-0.60$ & $+2.08$ \\
IoT-Wearables / Mitra    & alt1  & $-3.99$ & $-3.92$ & $+2.97$ & $-7.70$ & $-7.59$ & $+4.35$ & $-16.99$ & $-17.17$ & $+5.78$ \\
                         & alt2  & $-4.61$ & $-5.94$ & $-3.39$ & $-9.24$ & $-10.97$ & $-4.37$ & $-22.42$ & $-24.48$ & $-4.12$ \\
                         & alt3  & $-3.06$ & $-2.12$ & $+2.60$ & $-5.71$ & $-5.90$ & $+1.18$ & $-11.47$ & $-11.68$ & $+1.55$ \\
IoT-Wearables / TabPFN   & alt1  & $-4.02$ & $-3.44$ & $+4.37$ & $-7.76$ & $-7.84$ & $+4.21$ & $-17.08$ & $-17.74$ & $+4.68$ \\
                         & alt2  & $-4.63$ & $-6.21$ & $-1.95$ & $-9.28$ & $-12.55$ & $-3.48$ & $-22.47$ & $-26.57$ & $-6.79$ \\
                         & alt3  & $-3.02$ & $-3.01$ & $-0.27$ & $-5.63$ & $-5.29$ & $+0.01$ & $-11.28$ & $-11.47$ & $-0.22$ \\
    \bottomrule
  \end{tabular}
\end{table}

\begin{table}[h]
  \centering
  \scriptsize
  \setlength{\tabcolsep}{4pt}
  \caption{Drift of the value of time (\%) along $\lambda$: median and
  upper decile across test observations of the absolute relative
  deviation of the value of time per observation from the structural value, computed
  for each replicate and averaged over replicates.}
  \label{tab:mrs-drift}
  \begin{tabular}{llcccccc}
    \toprule
     & & \multicolumn{4}{c}{Median} & \multicolumn{2}{c}{Upper decile} \\
    \cmidrule(lr){3-6} \cmidrule(lr){7-8}
    Cell & Alternative & $\lambda{=}0.25$ & $0.5$ & $0.75$ & $1$ & $\lambda{=}0.25$ & $1$ \\
    \midrule
Swissmetro / Mitra       & train & $3$ & $6$ & $8$ & $11$ & $15$ & $43$ \\
                         & SM    & $17$ & $27$ & $33$ & $38$ & $29$ & $58$ \\
                         & car   & $4$ & $9$ & $13$ & $18$ & $18$ & $85$ \\
Swissmetro / TabPFN      & train & $3$ & $6$ & $8$ & $10$ & $2{,}393$ & $663$ \\
                         & SM    & $12$ & $20$ & $26$ & $30$ & $5{,}437$ & $1{,}662$ \\
                         & car   & $3$ & $5$ & $8$ & $11$ & $10$ & $45$ \\
LPMC / Mitra             & PT    & $3$ & $6$ & $9$ & $13$ & $12$ & $50$ \\
                         & drive & $7$ & $14$ & $20$ & $26$ & $24$ & $142$ \\
    \bottomrule
  \end{tabular}
\end{table}

\clearpage
\section{Ablation details}
\label{sec:ablation-details}

Table~\ref{tab:ablations} gives test accuracy for the adapter and
ablations A1, A2, A3 and A5 of Section~\ref{sec:ablations}, and
Figure~\ref{fig:graceful_degradation} shows ablation A4.

\begin{table}[h]
\centering
\scriptsize
\setlength{\tabcolsep}{4pt}
\renewcommand{\arraystretch}{0.9}
\caption{\textbf{Ablation summary}: test accuracy (\%), mean
$\pm$ std across $10$ bootstrap replicates, row-level splits, every
variant trained on the same cross-fitted $\mathbf{q}$ on training rows.
A4 (graceful degradation when the foundation model's context is restricted)
has a different shape (one row per context fraction) and is plotted
separately in Figure~\ref{fig:graceful_degradation}. TabPFN is not
reported on LPMC.}
\label{tab:ablations}
\resizebox{\textwidth}{!}{%
\begin{tabular}{lccccc}
\toprule
 & \multicolumn{2}{c}{Swissmetro} & LPMC & \multicolumn{2}{c}{IoT-Wearables} \\
\cmidrule(lr){2-3} \cmidrule(lr){4-4} \cmidrule(lr){5-6}
Variant & Mitra & TabPFN & Mitra & Mitra & TabPFN \\
\midrule
Adapter $V_{\text{struct}} + g(q)$ & 76.21 $\pm$ 0.24 & 74.93 $\pm$ 0.24 & 72.93 $\pm$ 0.11 & 66.06 $\pm$ 0.45 & 66.43 $\pm$ 0.51 \\
A1: $V_{\text{struct}} + \alpha \log q$ only & 75.34 $\pm$ 0.06 & 73.46 $\pm$ 0.19 & 72.08 $\pm$ 0.05 & 64.84 $\pm$ 0.73 & 64.77 $\pm$ 0.71 \\
A2: both terms $V_{\text{struct}} + \alpha \log q + g(q)$ & 76.23 $\pm$ 0.17 & 74.91 $\pm$ 0.19 & 72.91 $\pm$ 0.08 & 65.86 $\pm$ 0.24 & 66.28 $\pm$ 0.41 \\
A3: joint training & 77.05 $\pm$ 0.50 & 76.31 $\pm$ 0.25 & 74.03 $\pm$ 0.08 & 66.53 $\pm$ 0.34 & 66.57 $\pm$ 0.46 \\
A5: capacity sweep (range) & 76.14--76.35 & 74.81--75.13 & 72.85--72.93 & 65.92--66.27 & 66.04--66.43 \\
\bottomrule
\end{tabular}%
}
\end{table}

\begin{figure}[h]
  \centering
  \includegraphics[width=\textwidth]{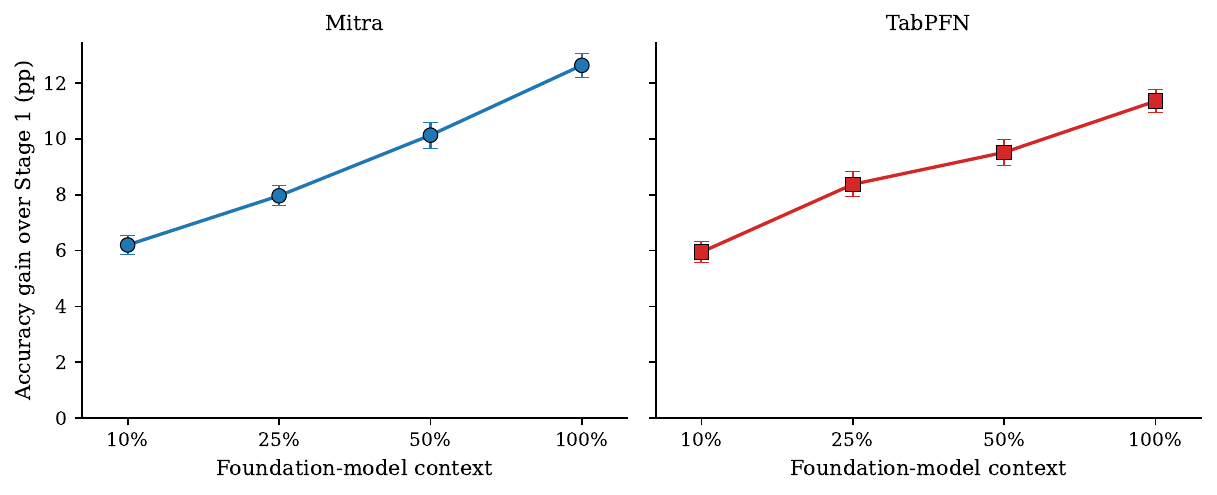}
  \caption{\textbf{Graceful degradation (A4)} on Swissmetro (row-level
  split): adapter accuracy gain over Stage~1 (pp) as the
  fraction of the foundation model's context is reduced. One panel per
  foundation model. Markers are means across $10$ bootstrap replicates
  with $95\%$ CI (cross-fitted protocol).}
  \label{fig:graceful_degradation}
\end{figure}


\end{document}